\documentclass[pdflatex,sn-mathphys-num]{sn-jnl}

\usepackage{graphicx}%
\usepackage{multirow}%
\usepackage{amsmath,amssymb,amsfonts}%
\usepackage{amsthm}%
\usepackage{mathrsfs}%
\usepackage[title]{appendix}%
\usepackage[dvipsnames]{xcolor}
\usepackage{textcomp}%
\usepackage{manyfoot}%
\usepackage{booktabs}%
\usepackage{algorithm}%
\usepackage{algorithmicx}%
\usepackage{algpseudocode}%
\usepackage{listings}%

\usepackage{url}
\usepackage{xr}

\usepackage{pgfplotstable}
\pgfplotsset{compat=1.18}
\usepackage{subcaption}

\usepackage{enumitem}

\usepackage{float}

\usepackage{comment} 

\usepackage{physics} 

\usepackage{empheq}

\usepackage{bbm}
\usepackage{bm} 

\usepackage{algpseudocode}

\usepackage{derivative}

\usepackage{cases}

\newcommand{\cD}{\mathcal{D}}
\newcommand{\cU}{\mathcal{U}}

\newcommand{\wdr}{w_{\rm rgrad}}

\newcommand{\cF}{\mathcal{F}}

\newcommand{\wnn}{W}

\newcommand{\hu}{\hat{u}}
\newcommand{\pibar}{\bar{\pi}}
\newcommand{\bx}{\vb x}
\newcommand{\by}{\vb y}
\newcommand{\bzero}{\vb 0}
\newcommand{\bW}{\vb W}

\theoremstyle{thmstyleone}%
\theoremstyle{thmstyletwo}%
\theoremstyle{thmstylethree}%

\newcommand{\LLO}{\mathcal{L}_{\text{LO}}} 

\newcommand{\Ltwo}[1]{L^{2}(\Omega)} 

\newcommand{\ubar}{\bar{u}} 
\newcommand{\wbar}{\overline{W}} 

\newcommand{\KLdiv}[2]{D_{\mathrm{KL}}\!\left(#1\;\big\|\;#2\right)}

\title[Bayesian BiLO with LoRA]{Bayesian BiLO: Bilevel Local Operator Learning for Efficient Uncertainty Quantification of Bayesian PDE Inverse Problems with Low-Rank Adaptation}

\author*[1]{\fnm{Ray Zirui} \sur{Zhang}}\email{rzhang8@wpi.edu}

\author[2]{\fnm{Christopher E.} \sur{Miles}}\email{chris.miles@uci.edu}

\author[3]{\fnm{Xiaohui} \sur{Xie}}\email{xhx@uci.edu}

\author*[2,4]{\fnm{John S.} \sur{Lowengrub}}\email{jlowengr@uci.edu}

\affil*[1]{\orgdiv{Department of Mathematical Sciences}, \orgname{Worcester Polytechnic Institute}, \orgaddress{\city{Worcester}, \postcode{01609}, \state{MA}, \country{USA}}}

\affil*[2]{\orgdiv{Department of Mathematics}, \orgname{University of California, Irvine}, \orgaddress{\city{Irvine}, \postcode{92697}, \state{CA}, \country{USA}}}

\affil[3]{\orgdiv{Department of Computer Science}, \orgname{University of California, Irvine}, \orgaddress{\city{Irvine}, \postcode{92697}, \state{CA}, \country{USA}}}

\affil[4]{\orgdiv{Department of Biomedical Engineering}, \orgname{University of California, Irvine}, \orgaddress{\city{Irvine}, \postcode{92697}, \state{CA}, \country{USA}}}

\abstract{Uncertainty quantification in PDE inverse problems is essential in many applications. Scientific machine learning and AI enable data-driven learning of model components while preserving physical structure, and provide the scalability and adaptability needed for emerging imaging technologies and clinical insights. We develop a Bilevel Local Operator Learning framework for Bayesian inference in PDEs (B-BiLO). At the upper level, we sample parameters from the posterior via Hamiltonian Monte Carlo, while at the lower level we fine-tune a neural network via low-rank adaptation (LoRA) to approximate the solution operator locally. B-BiLO enables efficient gradient-based sampling without synthetic data or adjoint equations and avoids sampling in high-dimensional weight space, as in Bayesian neural networks, by optimizing weights deterministically. We analyze errors from approximate lower-level optimization and establish their impact on posterior accuracy. Numerical experiments across PDE models, including tumor growth, demonstrate that B-BiLO achieves accurate and efficient uncertainty quantification.
}



\begin{document}
\maketitle

\section{Introduction}
Uncertainty quantification in computational models governed by partial differential equations (PDEs) 
is essential across scientific and engineering disciplines, including seismic imaging \cite{dengOpenFWILargeScaleMultiStructural2023,martinStochasticNewtonMCMC2012,yangSeismicWavePropagation2021,bui-thanhExtremescaleUQBayesian2012}, personalized medicine \cite{lipkovaPersonalizedRadiotherapyDesign2019,chaudhuriPredictiveDigitalTwin2023}, climate modeling \cite{senGlobalOptimizationMethods2013}, and engineering design \cite{cookOptimizationTurbulenceModel2019}. These fields frequently require solving PDE inverse problems within a Bayesian framework \cite{stuartInverseProblemsBayesian2010}, where the objective is to estimate the posterior distribution of model parameters given observed data. Typically, this process involves sampling PDE parameters and repeatedly solving the forward PDE. A related approach, PDE-constrained optimization, aims to identify a single optimal parameter estimate, thus offering computational efficiency but lacking the ability to quantify uncertainty.
Beyond the cost and challenges of solving PDEs, mathematical models may be misspecified due to limited scientific understanding and therefore require continual refinement as new data become available. 
Scientific machine learning offers a framework for learning components of PDE models directly from data while retaining physical structure. Moreover, the rapid pace of new imaging technologies and clinical insights demands computational tools that are fast, scalable, and adaptable—capabilities well aligned with modern AI/ML methods.
Such needs have driven growing interest in scientific machine learning methods that incorporate data-driven learning of model components.

The Bilevel Local Operator learning (BiLO) method has recently shown promise for solving PDE inverse problems in the constrained optimization framework \cite{zhangBiLOBilevelLocal2025a}. 
A brief review of BiLO is provided in the supplementary material (SM).
In this work, we introduce a novel framework, Bayesian BiLO (B-BiLO), which redesigns this bilevel structure for the more challenging setting of Bayesian inference. 
B-BiLO enables robust uncertainty quantification by integrating local operator learning with gradient-based MCMC sampling and transfer learning techniques for efficiency.

Addressing Bayesian PDE inverse problems involves two design choices: the \textbf{PDE solver} and the \textbf{sampling method}. PDE solutions can be computed using numerical methods (such as finite difference or finite element methods), physics-informed neural networks (PINNs)\cite{raissiPhysicsinformedNeuralNetworks2019}, or neural operators such as Fourier Neural Operator \cite{liFourierNeuralOperator2021} and DeepONet \cite{luLearningNonlinearOperators2021}. Each approach has its advantages and limitations. Numerical methods provide fast, accurate, and convergent solutions but can struggle with high-dimensional PDEs \cite{grossmannCanPhysicsinformedNeural2024}. PINNs offer a mesh-free approach, effective in some higher-dimensional PDE problems, and a unified framework for inverse problems. Neural operators, while requiring many synthetic numerical solutions to train, enable rapid evaluations and serve as PDE surrogates; however, their accuracy typically degrades when evaluated outside the distribution of the training dataset \cite{nelsenOperatorLearningMeets2025}.
For sampling, Markov Chain Monte Carlo (MCMC) \cite{hastingsMonteCarloSampling1970} methods can be either gradient-based (e.g., Hamiltonian Monte Carlo \cite{nealMCMCUsingHamiltonian2011,nealBayesianLearningNeural1996,girolamiRiemannManifoldLangevin2011}) or derivative-free (e.g., Ensemble Kalman Sampling \cite{garbuno-inigoInteractingLangevinDiffusions2020}). Gradient-based methods efficiently explore posterior distributions, but computing the gradient can be expensive (compared with gradient-free approaches) or impractical (e.g., when the PDE solver is a black box).

Numerical PDE solvers can be paired with gradient-based sampling methods. The \textbf{Adjoint Method} computes gradients of the likelihood function with respect to PDE parameters \cite{bui-thanhSolvingLargescalePDEconstrained2014}. While the adjoint method ensures accurate solutions and gradients, deriving and regularizing the adjoint equations for complex PDEs can be challenging and solving both the forward and adjoint equations numerically can be computationally intensive. When only the forward solver is available, it is more practical to use derivative-free sampling methods such as \textbf{Ensemble Kalman Sampling} (EnKS), which employ ensembles of interacting particles to approximate posterior distributions \cite{garbuno-inigoInteractingLangevinDiffusions2020}. Additionally, variants of \textbf{Ensemble Kalman Inversion} (EnKI) can efficiently estimate the first two moments in Bayesian settings \cite{huangEfficientDerivativefreeBayesian2022,huangIteratedKalmanMethodology2022,pensoneaultEfficientBayesianPhysics2024}.

Among neural PDE solvers, \textbf{Physics-Informed Neural Networks (PINNs)} \cite{raissiPhysicsinformedNeuralNetworks2019} have become widely adopted. 
In PINNs, the PDE residual is enforced as a soft constraint.
By minimizing a weighted sum of the data loss and the PDE residual loss, PINNs can fit the data, solve the PDE, and infer the parameters at the same time. 
For uncertainty quantification, Bayesian PINNs (BPINNs) \cite{yangBPINNsBayesianPhysicsinformed2021} represent PDE solutions using Bayesian neural networks, and use Hamiltonian Monte Carlo (HMC) \cite{duaneHybridMonteCarlo1987,nealBayesianLearningNeural1996} to sample the joint posterior distributions of PDE parameters and neural network weights. 
Derivative-free methods, such as EnKI, can also be effectively combined with neural PDE solvers, as demonstrated in PDE-constrained optimization problems \cite{guthEnsembleKalmanFilter2020} and Bayesian PINNs \cite{pensoneaultEfficientBayesianPhysics2024}.

\textbf{Neural Operators (NOs)} approximate the PDE solution operator (parameter-to-solution map), providing efficient surrogate models for forward PDE solvers \cite{kovachkiNeuralOperatorLearning2022}. 
Once trained, these surrogates integrate seamlessly with Bayesian inference frameworks, leveraging rapid neural network evaluations for posterior distribution sampling using both gradient-based and derivative-free methods \cite{pathakFourCastNetGlobalDatadriven2022,luMultifidelityDeepNeural2022,maoPPDONetDeepOperator2023}. Prominent examples include the Fourier Neural Operator (FNO) \cite{liFourierNeuralOperator2021,liPhysicsInformedNeuralOperator2024,whitePhysicsInformedNeuralOperators2023,akyildizEfficientPriorCalibration2025}, Deep Operator Network (DeepONet) \cite{luLearningNonlinearOperators2021,wangLearningSolutionOperator2021}, and In-Context Operator (ICON) \cite{yangIncontextOperatorLearning2023}, among others \cite{oleary-roseberryDerivativeInformedNeuralOperator2024,molinaroNeuralInverseOperators2023}. Nevertheless, neural operator performance heavily depends on the training dataset, potentially limiting accuracy in PDE solutions and consequently degrading posterior distribution fidelity \cite{caoResidualbasedErrorCorrection2023}.

Our framework integrates key practical advantages from existing sampling and PDE-solving methods. We employ gradient-based MCMC, which generally explores the posterior distribution more efficiently than derivative-free approaches. The PDE solution and the corresponding likelihood gradients with respect to parameters are computed using a neural solver within a novel bilevel optimization framework. This design is flexible across PDE models, and does not require either deriving the adjoint equations or curating pretraining datasets. Our methodology improves accuracy through strong enforcement of PDE constraints without balancing data fitting and PDE residuals as in PINNs, and enhances efficiency by optimizing network weights deterministically rather than sampling in a high-dimensional weight space as in Bayesian neural networks.
The key contributions of this work are as follows:
\begin{itemize}  
  \item We introduce a Bayesian Bilevel local operator (B-BiLO) learning approach that enforces strong PDE constraints, leading to improved accuracy in uncertainty quantification and parameter inference.
  \item We avoid direct sampling in the high-dimensional space of Bayesian neural networks, leading to more efficient sampling of unknown PDE parameters or functions.
  \item We apply Low-Rank Adaptation (LoRA)\cite{huLoRALowRankAdaptation2021} to enhance both efficiency and speed in fine-tuning and sampling, significantly reducing computational cost while maintaining accuracy. 
  \item We prove there is a direct link between the tolerance for solving the lower level problem  to approximate the local solution operator and the accuracy of the resulting uncertainty quantification. In particular, we estimate the error in the posterior distribution due to inexact minimization of the lower level problem.
\end{itemize}  


\section{Results}
\label{s:Results}

\subsection{Method Overview}
This section provides an overview of B-BILO. Full methodological details and theoretical analysis are provided in Section \ref{method} and the Supplementary Material (SM).

\paragraph{PDE Bayesian Inverse Problem}
We consider problems governed by partial differential equations (PDEs). 
Let $u:\Omega \to \mathbb{R}$ be a function (the state variable) defined on a domain $\Omega \subset \mathbb{R}^d$. 
The governing PDE, together with boundary and initial conditions, can be written compactly as
\[
\mathcal{F}[u,\theta](\bx) = \bzero, \qquad \bx \in \Omega,
\]
where $\theta \in \mathbb{R}^m$ are the PDE parameters and $\mathcal{F}$ is the associated differential operator, defined by
\begin{equation}
  \cF[u,\theta](\bx) := F(\cD^k u(\bx), \dots, \cD u(\bx), u(\bx), \theta), \qquad \bx\in\Omega,
\end{equation}
with $\mathcal{D}^k$ denoting the $k$-th order derivative operator. 
The solution to the PDE depends implicitly on $\theta$; we denote this solution by $u(\cdot,\theta)$.

Let $P(\theta)$ be the prior distribution of the PDE parameters. Let $D_u$ be the observed data, which is typically a noisy observation of the state $u$ at a set of measurement points. 
The likelihood $P(D_u|\theta)$ measures the probability of observing the data $D_u$ given the PDE parameters $\theta$. Since it is usually a functional of the state $u$, thus we also denote it as $\ell[u]$.
For example, if we assume $D_u = \{\hu_i\}_{i=1}^N$ with $\hu_i = u(\bx_i,\theta) + \eta_i$, where $\eta_i$ is Gaussian noise of mean 0 and variance $\sigma_d^2$,
then the likelihood of the observed data $D_u$ given the PDE parameters $\theta$ is
\begin{equation}
  \ell[u] = \prod_{i=1}^{N} \frac{1}{\sqrt{2\pi \sigma_d^2}} \exp\left( -\frac{1}{2\sigma_d^2} \left| u(\bx_i) - \hu_i\right|^2 \right)
\end{equation}
which is $P(D_u|\theta)$ if $u$ is the solution at $\theta$.
Let $P(\theta|D_u)$ be the posterior distribution of the PDE parameters.
By Bayes' theorem, $P(\theta|D_u) \propto P(D_u|\theta) P(\theta)$. 
The potential energy, $\cU[u,\theta]$, which depends on both the state $u$ and the PDE parameters $\theta$,
is defined as the negative of the log-posterior distribution
\begin{equation}
  \cU[u, \theta] = -\log P(\theta|D_u) = -\log \ell [u] - \log P(\theta)
  \label{potential energy}
\end{equation}
We also denote $U(\theta) = \cU[u(\cdot,\theta), \theta]$, as the potential energy ultimately depends on $\theta$ if the state $u$ is the solution at $\theta$.
Sampling the posterior distribution with gradient-based MCMCs, such as Langevin Dynamics (LD) or Hamiltonian Monte Carlo (HMC), requires computing the gradient of the potential energy $\nabla_{\theta} U(\theta)$ to effectively explore the posterior distribution.

\paragraph{Bayesian Bilevel Local Operator Learning (B-BiLO).}
We consider functions of the form $u(\bx, \theta)$ that depend on both the spatial-temporal variable $\bx$ and the PDE parameters $\theta$.
We call $u(\cdot,\theta)$ a local operator at $\theta$ if
\begin{align*}
  &\cF[u(\cdot,\theta),\theta](\bx) = \bzero  \quad(\text{Condition 1}),\\
  &d_\theta \cF[u(\cdot,\theta),\theta](\bx) = \bzero \quad(\text{Condition 2}),
\end{align*}
for all $x\in\Omega$, where $d_\theta$ is the total derivative with respect to $\theta$.
Condition 1 requires the residual of the PDE to be zero, i.e., that $u(\cdot,\theta)$ is a solution of the PDE at $\theta$.
Condition 2 requires the total derivative of the PDE residual with respect to $\theta$ to be zero, ensuring that the derivative of $u(\cdot,\theta)$ with respect to $\theta$ remains consistent with the PDE constraints.
This ensures that small perturbations in $\theta$ lead to small perturbations in $u(\cdot,\theta)$ that still approximately satisfy the PDE constraints. In particular, if we consider a local perturbation $\theta \to \theta+\delta\theta$, a formal Taylor expansion of $\mathcal F$ about $\theta$ is zero to second order in $\delta\theta$ (e.g., see \cite{zhangBiLOBilevelLocal2025a}).

In this work, we represent the local operator by a neural network $u(\bx, \theta; \wnn)$, where $\wnn$ are the weights of the neural network.
The two conditions for $u$ to be a local operator at $\theta$ lead to the definition of the local operator loss $\LLO$:
\begin{equation}
  \LLO(\theta, \wnn) = \norm{\cF[u(\cdot,\theta;W),\theta]}^2_2 + \wdr \norm{d_{\theta} \cF[u(\cdot,\theta;W),\theta]}^2_2,
  \label{eq:llo}
\end{equation}
We call the first term the residual loss and the second term the residual-gradient loss. Both terms can be approximated using the Mean Squared Error (MSE) at a set of collocation points.
The weight $\wdr> 0$ is a scalar hyperparameter that controls the relative importance of the residual-gradient loss. 
In principle, both terms can be minimized to nearly zero.
For fixed $\theta$, by minimizing the local operator loss $\LLO(\theta, \wnn)$ with respect to the weights $\wnn$, we approximate the local operator $u(\bx, \theta)$ around $\theta$.

We consider the following Bayesian inference problem with a bilevel structure:
\begin{equation}
  \label{eq:bi_bayes}
  \begin{cases}
    \theta \sim \exp \left(-\cU[u(\cdot,\theta;W^*(\theta)),\theta]\right)\\
    \wnn^*(\theta) = \arg\min_{\wnn} \LLO(\theta, \wnn)\\
  \end{cases}
\end{equation}
At the upper level, we sample the PDE parameters $\theta$ from the posterior distribution using Hamiltonian Monte Carlo (HMC)~\cite{nealMCMCUsingHamiltonian2011}.
At the lower level, we train a network to approximate the local operator by minimizing the local operator loss with respect to the weights of the neural network.

\paragraph{Algorithm and Theoretical Analysis}
We use HMC with the leapfrog integrator to sample the posterior distribution at the upper level. HMC requires computing the gradient of the potential energy $\nabla_{\theta} U(\theta)$ at various values of $\theta$, which require solving the lower-level problem to some tolerance $\epsilon$.
Detailed algorithms are provided in Sec.~\ref{method:algo}.
Under mild stability/smoothness assumptions, we show that the error in the gradient $\nabla_{\theta} U(\theta)$ and the error in the posterior distribution, both due to inexact minimization of the lower level problem, are $O(\epsilon)$.
The proofs are provided in Sec.~\ref{method:theory} and SM.

\paragraph{Advantages over existing methods}
Many existing methods, such as Bayesian PINNs (BPINNs) \cite{yangBPINNsBayesianPhysicsinformed2021,linkaBayesianPhysicsInformed2022}, represent the PDE solution using a Bayesian neural network $v(\bx;W)$,
(note that $\theta$ is not an input to the network), 
where $W$ follows a prior distribution $P(W)$. Uncertainty in $W$ induces uncertainty in the PDE residual $D_f = \{(\bx_f^{(i)}, \eta_f)\}_{i=1}^{N_f}$, $\eta_f \sim \mathcal{N}(0, \sigma_f^2)$, and give rise to a likelihood function for the residual term:
\begin{equation}
  P(D_f|W, \theta) = \prod_{i=1}^{N_f} \frac{1}{\sqrt{2\pi \sigma_f^2}} \exp\left( -\frac{1}{2\sigma_f^2} \left| \cF[v(\bx_f^{(i)}; W), \theta] \right|^2 \right)
  \label{eq:residual}
\end{equation}
which serves as a soft constraint to enforce the PDE, and depends on $\sigma_f$.
Assuming $\theta$ and $W$ are independent, the following joint posterior is sampled using HMC.
\begin{equation}
  P(W, \theta|D_u, D_f) \propto P(D_u|W) P(D_f|W, \theta) P(W) P(\theta).
\end{equation}
These would require sampling the challenging high-dimensional joint posterior of the PDE parameters $\theta$ and network weights $\wnn$.
Making $\sigma_f$ small enforces the PDE strongly, but also makes the loss ill-conditioned and leads to inefficient sampling.
In contrast, B-BiLO avoids these issues by optimizing $W$ deterministically for each $\theta$, and thus only the low-dimensional parameter space $\theta$ is sampled. 
This approach removes the need for priors on weights, avoids the difficulties setting an artificial residual noise parameter, and enables more stable and efficient exploration of the posterior. 
A detailed review of BPINNs and it's challenges are provided in SM.

\paragraph{Efficient Fine-Tuning via LoRA}
To increase efficiency, we apply Low-Rank Adaptation (LoRA) \cite{huLoRALowRankAdaptation2021} to reduce the number of trainable parameters at the lower level.
Let a hidden-layer weight be $W\in \mathbb{R}^{p\times p}$. 
Across parameters $\theta$ and an initial guess $\theta_0$, BiLO updates only a low-rank increment
\[
W^*(\theta) \approx W^*(\theta_0) + A B,\qquad A\in \mathbb{R}^{p\times r},\; B\in \mathbb{R}^{r\times p},\; r\ll p,
\]
while keeping the $W^*(\theta_0) $ fixed.
We refer to this procedure as ``LoRA($r$)'', denoting LoRA fine-tuning with rank $r$ (Algorithm \ref{alg:loraft}), whereas ``Full FT'' denotes full fine-tuning of all weights (Algorithm \ref{alg:fullft}).
Details of the neural network architecture and implementation of LoRA are provided in Methods.~\ref{method:arch} and \ref{method:lora}.

\paragraph{Schematic of B-BiLO with LoRA}
A schematic of the B-BiLO framework is shown in Figure~\ref{f:schema}.
The figure is based on a model boundary value problem $-\theta u_{xx}=\sin(\pi x)$ with the boundary condition $u(0) = u(1) = 0$. 
In this case, the full operator is given by $u(x,\theta) = \sin(\pi x) / (\pi^2 \theta)$, which solves the PDE for all $\theta>0$.
Figure~\ref{f:schema} panel (A) illustrates the upper-level problem: the goal is to sample the PDE parameters $\theta$ from the posterior distribution. For gradient-based MCMC methods, we need to compute the gradient of the potential energy $\nabla_{\theta} U(\theta)$ at various values of $\theta$ (orange arrows).
Therefore, at any fixed $\theta$, we need access to the local behavior of the potential energy, shown as the blue strip near $\theta_1$ and $\theta_2$. 

In Figure~\ref{f:schema} panel (B), the gray surface is the full PDE operator $u(x,\theta)$. 
At some particular $\theta_1$, the local operator $u(x,\theta_1)$ approximates a small neighborhood of the full operator near $\theta_1$ (blue surface). 
The condition of the local operator is that the residual and the gradient of the residual with respect to $\theta$ are both zero at any fixed $\theta$.
Knowing the PDE solution in the neighborhood of $\theta_1$ allows us to compute the gradient of the potential energy at $\theta_1$.

Figure~\ref{f:schema} panel (C) illustrates the lower-level problem and LoRA: we train a neural network to approximate the local operator at different $\theta$, e.g. $\theta_1$ and $\theta_2$. The local operator at each $\theta_i$ is approximated by a neural network $u(x,\theta_i; \wnn)$ and is obtained by minimizing the local operator loss $\LLO(\theta_i, \wnn)$. Instead of updating all the weights $\wnn$ at each $\theta_i$, we use LoRA to update the weights for different $\theta_i$, which significantly reduces the number of trainable parameters and thus the computational cost.

\begin{figure}[!htb]
  \centering
  \includegraphics[keepaspectratio=true, width=\textwidth]{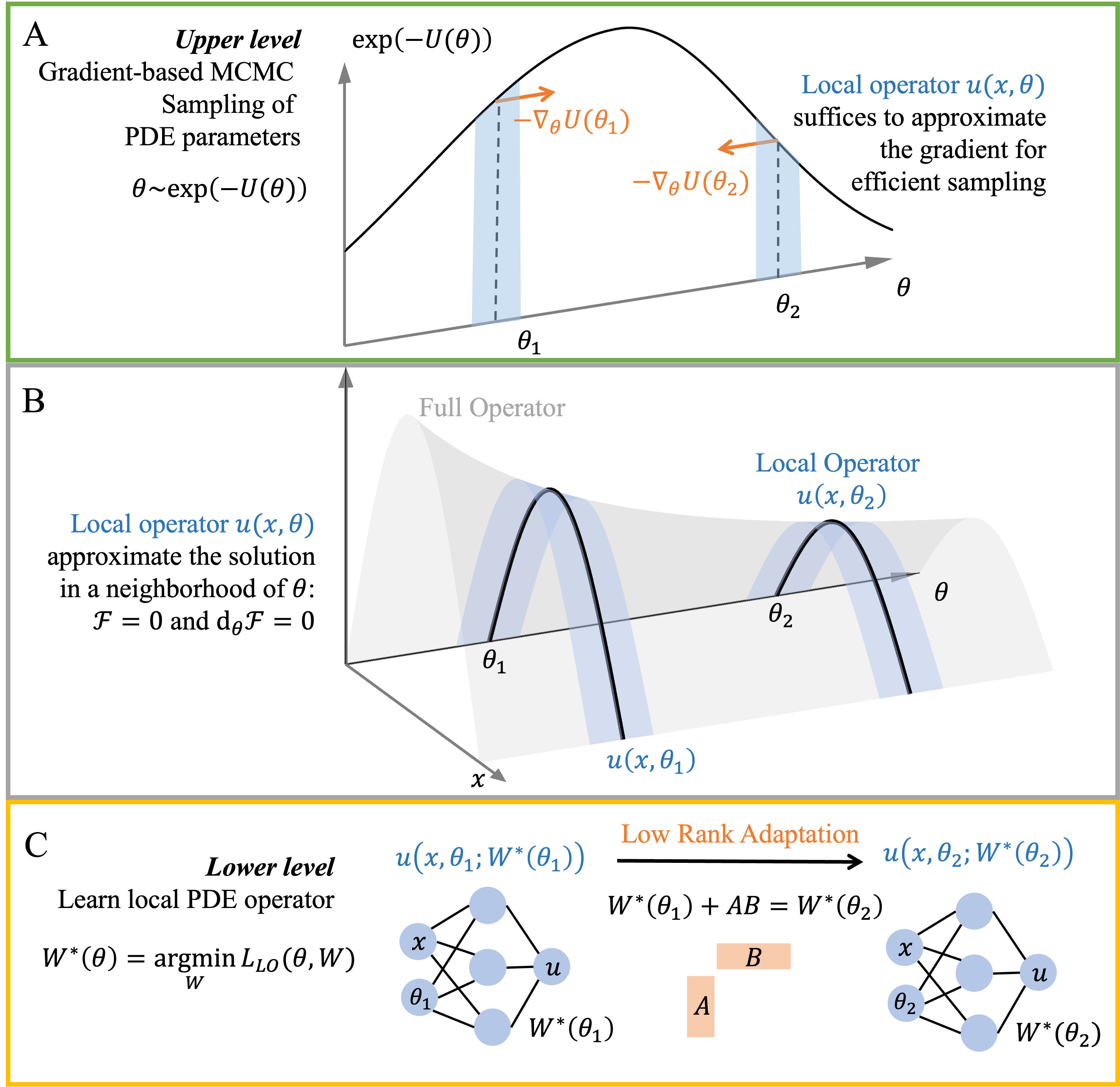}
  \caption{A schematic for solving the Bayesian PDE inverse problem using BiLO and LoRA. 
  (A) At the upper level, we sample the PDE parameters $\theta$ from the posterior distribution using gradient-based MCMC. 
  (B) The full operator $u(x, \theta)$ (gray surface) and its local approximation $u(x, \theta_1)$ and $u(x, \theta_2)$ (blue surface), which satisfies vanishing residual and residual-gradient conditions. The local operator suffices to compute the gradient of the potential energy $U(\theta)$ at $\theta_1$ and $\theta_2$.
  (C) At the lower level, we train the neural network to approximate the local operator (blue surface) at different $\theta$. As $\theta$ changes, only low rank adaptation (LoRA) is used to update the neural network weights $\wnn$. 
  }
  \label{f:schema}
\end{figure}

\subsection{Example 1: Nonlinear Poisson Equation}
\label{ss:nonlinpoi}
In this section, we compare B-BiLO with BPINN \cite{yangBPINNsBayesianPhysicsinformed2021} on a nonlinear Poisson problem to demonstrate the advantages of B-BiLO: physical uncertainty quantification, stable sampling with larger step sizes, and reduced computational cost. 
We consider the 1D nonlinear Poisson equation \cite{yangBPINNsBayesianPhysicsinformed2021}:
\begin{equation}
    \lambda u_{xx}(x) + k \tanh(u(x)) = f(x), \quad x\in(-0.7, 0.7)
\end{equation}
where $\lambda = 0.01$, and both $f$ and the boundary conditions are derived from the ground-truth solution $u^{\rm GT}(x) = \sin^3(6x)$ with $k^{\rm GT} = 0.7$.
The goal is to infer $k$ from observations of $u$, with the prior that $k$ is uniformly distributed on $[0.4, 1]$.
A key feature of this problem is that the solutions are odd: both $u$ and $-u(-x)$ solve the PDE, and thus $u(0) = 0$ for all $k$ in the range. Therefore, there should be no uncertainty at the origin.

\begin{figure}[!htb]
  \centering
  \includegraphics[keepaspectratio=true, width=\textwidth]{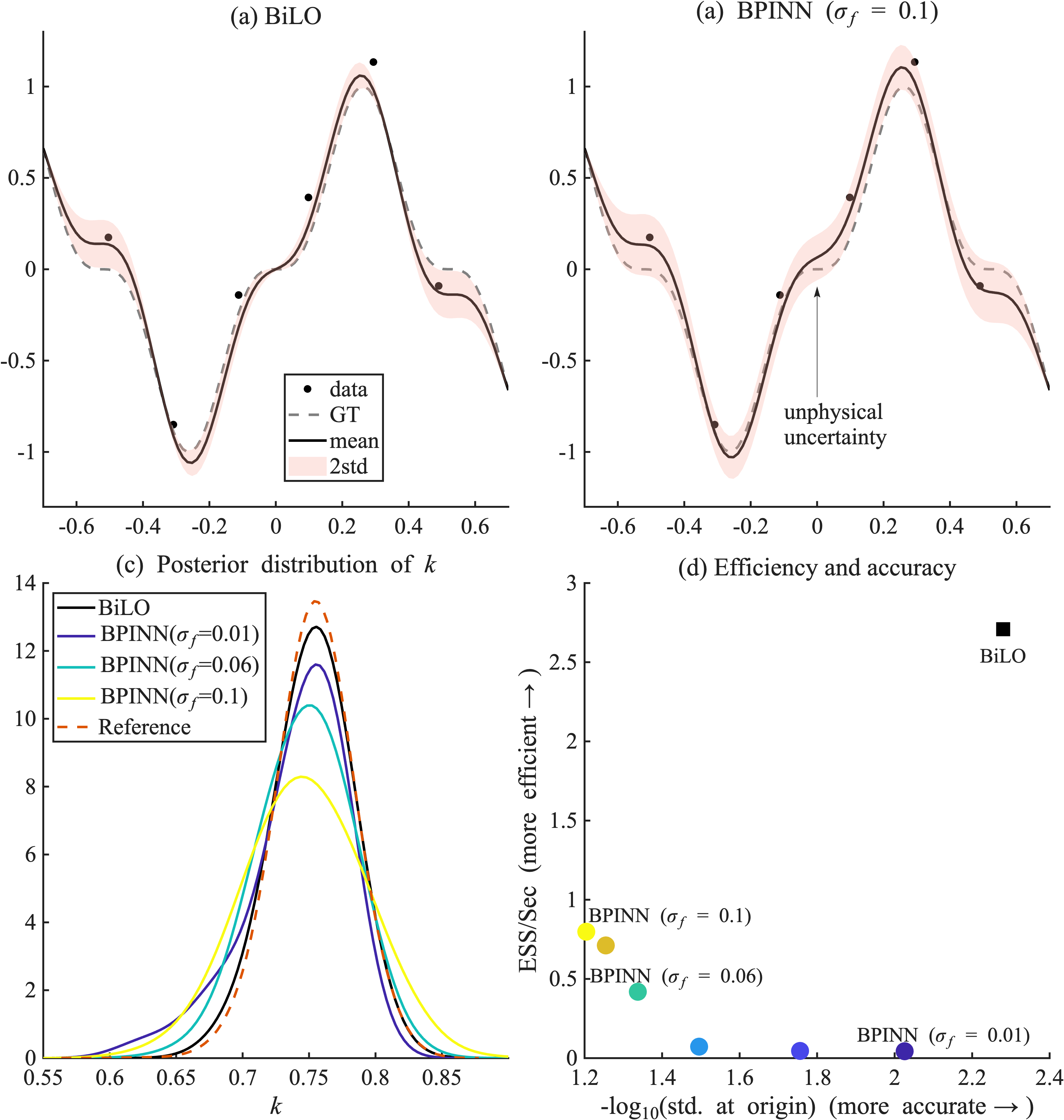}
  \caption{ 
  Inference results for the nonlinear Poisson problem.
  (a) Noisy data, GT solution, mean and std of inferred $u$ using B-BiLO
  (b) results using BPINN with $\sigma_f=0.1$. 
  (c) Posterior distribution of the PDE parameter $k$ using BiLO and BPINN with different $\sigma_f$, compared with the reference (Metropolis-Hastings with numerical solver).
  (d) Accuracy-efficiency comparison:
  x-axis shows the -log of the std. at the orgin, larger values indicate better physical uncertainty quantification.
  y-axis shows the effective sample size (ESS) per second of wall time, larger values indicate more efficient sampling.
  As $\sigma_f$ decreases, BPINN becomes more accurate but less efficient.
  Overall B-BiLO is more accurate and efficient at 
  uncertainty quantification.}
  \label{f:nonlin_uq}
\end{figure}

Figure \ref{f:nonlin_uq} compares B-BiLO and BPINN on the nonlinear Poisson problem.
Panel (a) shows the noisy data, GT, mean and 2 std. region of the inferred solution $u$ from B-BiLO. The samples of $u$ satisfy the PDE accurately—evidenced by nearly zero posterior standard deviation at the origin.
Panel (b) shows BPINN with $\sigma_f=0.1$, where the nonzero uncertainty at the origin indicates that individual samples are not solving the PDE accurately.
Panel (c) presents the posterior of the PDE parameter $k$.
The reference distribution is obtained using Metropolis–Hastings sampling (which is generally less efficient than HMC) combined with a accurate numerical solver.
B-BiLO closely matches the reference distribution, while BPINN remains accurate only for small $\sigma_f$; larger $\sigma_f$ reduces fidelity.
Panel (d) compares accuracy and efficiency. The x-axis (‒$\log_{10}$ of the std at origin) measures physical consistency—larger values indicate more accurate PDE solution.
While the y-axis measures the efficiency using the Effective Sample Size (ESS, which quantifies the number of independent samples) per second.
For each $\sigma_f$, we cross validated the step size $\delta t$ = {0.01, 0.005, 0.001, 0.0005} and report the best ESS/sec.
As $\sigma_f$ decreases (color from yellow to blue), BPINN becomes more accurate but less efficient, due to the stability constraint on the step size.
B-BiLO achieves both high accuracy and high efficiency simultaneously, without tuning $\sigma_f$.
Additional visualizations are provided in SM.

Figure~\ref{f:nonlin_lora} illustrates the accuracy and efficiency of B-BiLO with LoRA of varying ranks.
Panel (a) compares the posterior distribution of $k$ for LoRA ranks 4 and 8 against Full FT, demonstrating that LoRA does not compromise accuracy.
Panel (b) reports the relative efficiency of LoRA (rank 8) versus Full FT, plotted as ratios with respect to Full FT. The x-axis shows the number of neurons per layer (depth fixed at 6), indicating increasing network size. The right y-axis shows the relative maximum memory usage, which remains consistently below 1, confirming LoRA’s reduced memory footprint. The left y-axis (blue line) shows the relative average wall     time per sample. For small networks, LoRA is slower than Full FT (ratio above 1), while for large networks it becomes faster (ratio below 1). This reflects a trade-off: LoRA reduces the number of trainable parameters but introduces computational overhead and may require more iterations to reach comparable accuracy. 
As network size grows, the speedup due to reduced trainable parameters outweighs the cost of the overhead and additional iterations.
\begin{figure}[!htb]
  \centering
  \includegraphics[keepaspectratio=true, width=\textwidth]{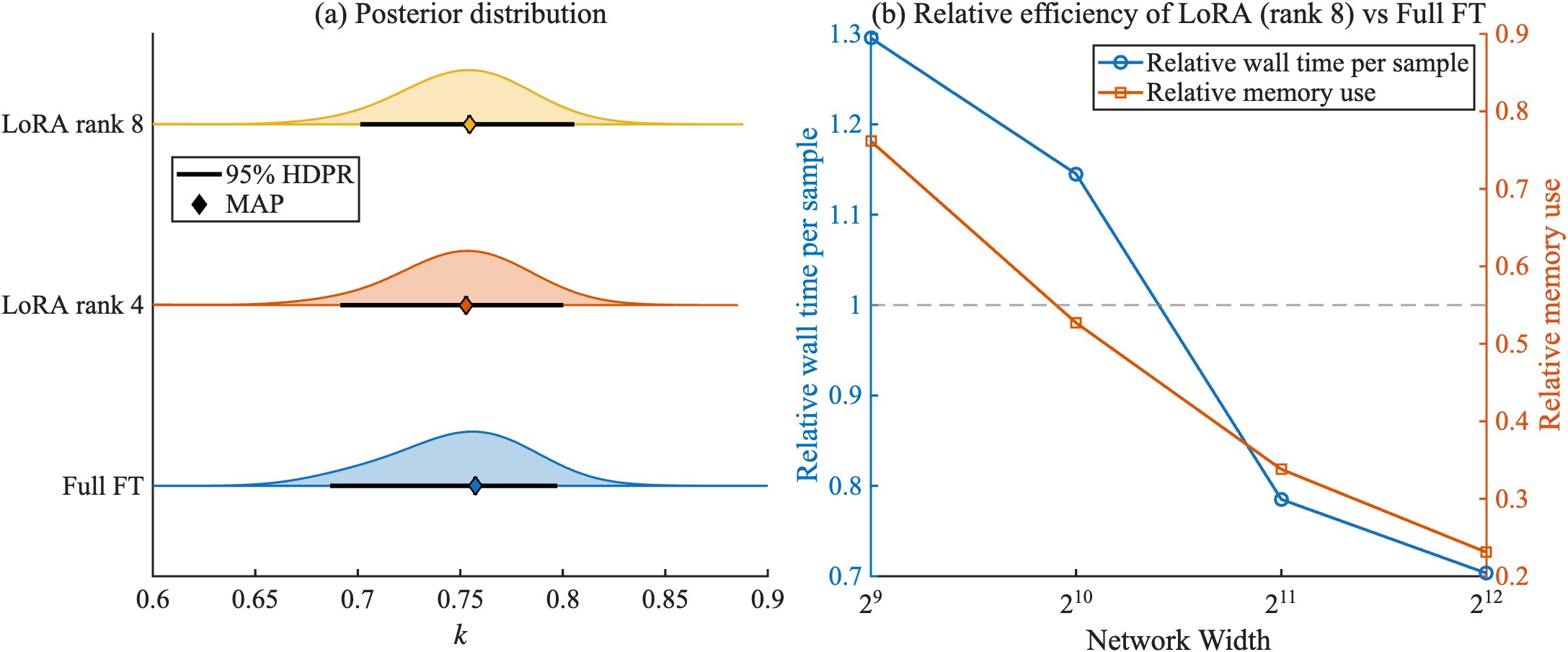}
  \caption{ 
    Accuracy and efficiency of B-BiLO with LoRA of different ranks.
(a) Comparison of the posterior distribution of $k$ obtained by B-BiLO with LoRA of different ranks and Full FT. Diamond markers denote the MAP estimates, and bars indicate the 95\% highest-density posterior region (HDPR).
(b) Relative efficiency of LoRA (rank 8) as ratios with respect to Full FT. The x-axis shows the number of neurons per layer; the left y-axis (blue line) shows the relative average time per sample, and the right y-axis shows the relative maximum memory usage during sampling.}
  \label{f:nonlin_lora}
\end{figure}

\subsection{Example 2: Inferring Patient-Specific Tumor Growth Parameters from MRI Data}
We consider a real-world application of B-BiLO for uncertainty quantification of patient-specific parameters GBM growth models using patient MRI data in 2D \cite{zhangPersonalizedPredictionsGlioblastoma2025,balcerakIndividualizingGliomaRadiotherapy2025,ezhovLearnMorphInferNewWay2023,scheufeleFullyAutomaticCalibration2021,lipkovaPersonalizedRadiotherapyDesign2019}. 
Glioblastoma (GBM) is a highly aggressive brain tumor with poor prognosis, where tumor cells infiltrate beyond lesions visible on MRI, limiting the effectiveness of current treatment strategies and leading to recurrence, thereby underscoring the need for personalized prediction. 
The infiltration and proliferation of tumor cells can be  modeled by reaction–diffusion equations \cite{swansonQuantitativeModelDifferential2000}.
Here, we use the formulation presented in \cite{zhangPersonalizedPredictionsGlioblastoma2025}. Let $\Omega$ denote the 2D brain domain derived from MRI data, and the normalized tumor cell density $u(\vb x, t)$ satisfy the Fisher–KPP equation with Neumann boundary conditions:
\begin{equation} \label{eq:gbm}
    \pdv{u}{t} = D \bar{D} \grad \cdot (P(\vb x) \grad u) + \rho \bar{\rho} u (1-u) \quad \text{in }  \Omega 
\end{equation}
where $P$ depends on the tissue distribution (e.g., white and grey matter) obtained from the MRI data, 
$\bar{D}$ and $\bar{\rho}$ are known patient-specific characteristic parameters based on the data,
and $D$ and $\rho$ are the unknown nondimensionalized parameters that we aim to infer.

We consider two regions of interest in the tumor, the whole tumor (WT) region and the tumor core (TC) region.
Let $\hat{\by}^{s}$, $s\in \{\rm WT, TC\}$ be indicator functions of the WT and TC regions, which can be obtained from the MRI data and serves as the observational data in the inverse problem.
We define our predicted segmentations to be regions in which the tumor cell density $u$ at the nondimensional $t=1$ lies above a certain threshold $u_c^s$.
The likelihood function is defined based on error between the predicted segmentation $\by^s$ and the observed segmentation $\hat{\by}^s$.
We aim to find the posterior distribution of the parameters $(D, \rho, u_c^{\rm WT}, u_c^{\rm TC})$.
Details of the setup and validation with synthetic data are provided in the SM (see also \cite{zhangPersonalizedPredictionsGlioblastoma2025}).

Figure \ref{f:gbm} summarizes the inference results for one patient case.
In panel (a), the posterior distributions of the parameters show that the model parameters are well identified with moderate uncertainty, and the results from B-BiLO with LoRA rank 8 closely match those from Full FT. Panel (b) shows the posterior mean and standard deviation of the tumor cell density $u$, respectively.  In panel (c), the MAP segmentation is seen to closely match the observed WT and TC regions on the MRI, and the predicted infiltration margin (where $u=1\%$) extends slightly beyond the visible lesion, suggesting possible microscopic invasion. In panel (d), the standard deviation of the tumor cell density $u$ is shown. Because the data are noisy and the boundary between the TC and WT regions is poorly defined, the estimated uncertainty is higher in this transition zone.
Although no direct ground truth exists for the true tumor cell density in patients, the predictive capability of the model can be indirectly validated using recurrence data \cite{zhangPersonalizedPredictionsGlioblastoma2025,balcerakIndividualizingGliomaRadiotherapy2025,balcerakPhysicsRegularizedMultiModalImage2024}.
This example illustrates the capability of B-BiLO to handle real patient data for risk-aware treatment planning; applications to full 3D patient datasets are left for future work.

\begin{figure}[!htb]
  \centering
  \includegraphics[keepaspectratio=true, width=\textwidth]{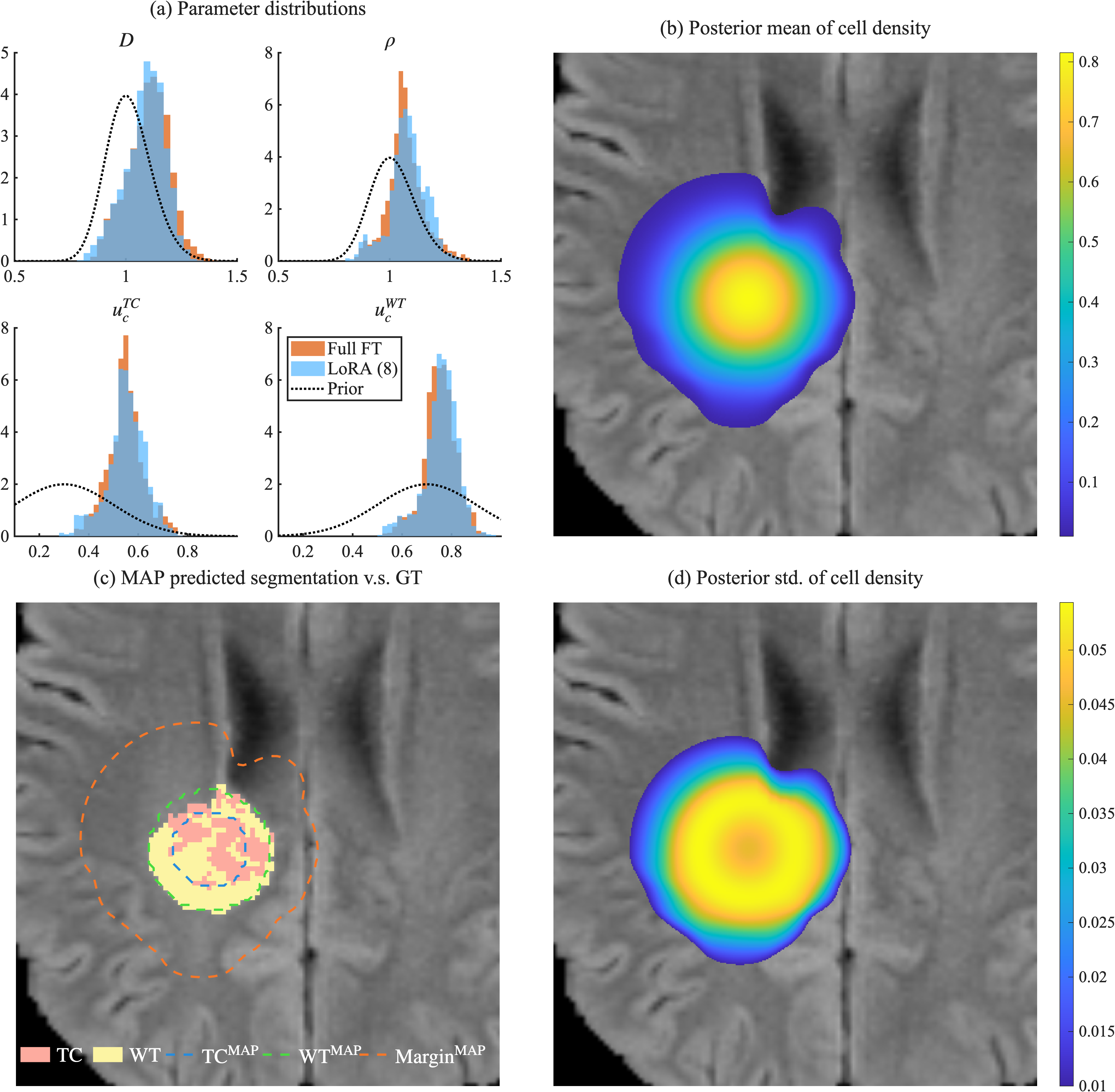}
  \caption{
Inference of patient-specific GBM growth parameters using B-BiLO with LoRA rank 8.
(a) Posterior distribution of the parameters $D, \rho, u_c^{\rm WT}, u_c^{\rm TC}$ and the prior (dashed), compared with Full FT.
(b) Posterior mean of the tumor cell density $u$.
(c) Ground truth MRI data with segmented WT and TC regions (filled), the MAP predicted segmentation (contour), and the predicted infiltration margin (where $u=1\%$).
(d) Posterior standard deviation of the tumor cell density $u$.
}
  \label{f:gbm}
\end{figure}

\subsection{Example 3: Inferring Stochastic Rates from Particle Data}
\label{ss:pointproc}

We consider a problem motivated by inferring the dynamics of gene expression from static images of mRNA molecules in cells \cite{milesInferringStochasticRates2024,milesIncorporatingSpatialDiffusion2025}.
The steady-state spatial distribution of these molecules can be modeled by the following boundary value problem with a singular source:
\begin{equation} \label{eq:pde_pointproc}
     \frac{d^2u}{dx^2} + \lambda \delta (x-z) - \mu u = 0, \quad u(0) = u(1) = 0.
\end{equation}
Here, the solution $u(x)$ represents the intensity of a spatial Poisson point process describing the locations of mRNA particles in a simplified 1D domain, and $\delta(x-z)$ is the Dirac delta function representing a point source of mRNA at location $z$.
The parameter $\lambda$ is the dimensionless birth rate (transcription) of mRNA at a specific gene site $z$, while $\mu$ is a degradation rate, and the boundary condition describes export across the nuclear boundary.

Given $M$ snapshots of mRNA locations $\{q_i^j\}$ for $j = 1, \ldots, M$, $i = 1, \ldots, N_j$, where $N_j$ is the number of particles in snapshot $j$, the goal is to infer the kinetic parameters $\lambda$ and $\mu$. Since the particle locations follow a Poisson point process, we can construct a likelihood function directly from their positions. The resulting log-likelihood is \cite{milesInferringStochasticRates2024}:
\begin{equation*}
  l(\lambda,\mu; q_i^j) =  M \int_\Omega u(x; \lambda, \mu) dx - \sum_{j=1}^M \sum_{i=1}^{N_j} \log u(q_i^j; \lambda, \mu).
\end{equation*}
This particular likelihood function, which is standard for spatial Poisson processes, consists of two terms: the sum of the log-intensity at each observed particle's location and a term penalizing the total expected number of particles in the domain, $\int_\Omega u(x)dx$. 
The solution to the PDE is continuous, but its derivative is discontinuous at the source $z$; we handle this singularity by adapting the cusp-capturing PINN \cite{tsengCuspcapturingPINNElliptic2023}. The details are explained in the SM.

In Fig.\ref{f:pp}(a), we show the 90\% highest-density posterior region (HDPR) and the MAP estimates of $\lambda$ and $\mu$ obtained using the reference method, B-BiLO with LoRA (rank 4), and B-BiLO with Full FT.
Both fine-tuning strategies closely match the reference, with relative errors in the MAP estimates and HDPR area below 4\% (details in SM).
The slight overestimation of $\lambda_{\rm MAP}$ and $\mu_{\rm MAP}$ likely arises from the inexact PDE solution, while the small discrepancies between Full FT and LoRA are attributable to variability in finite MCMC sampling.
The inset shows the HDPR and MAP of the solution $u(x)$, which has a cusp at the source location $z=0.5$ and is accurately captured by B-BiLO.
In Fig.~\ref{f:pp}(b), we show the ground truth (GT), prior distributions, and posterior distributions of $\lambda$ and $\mu$.

\begin{figure}[!htb]
  \centering
  \includegraphics[keepaspectratio=true,width=\textwidth]{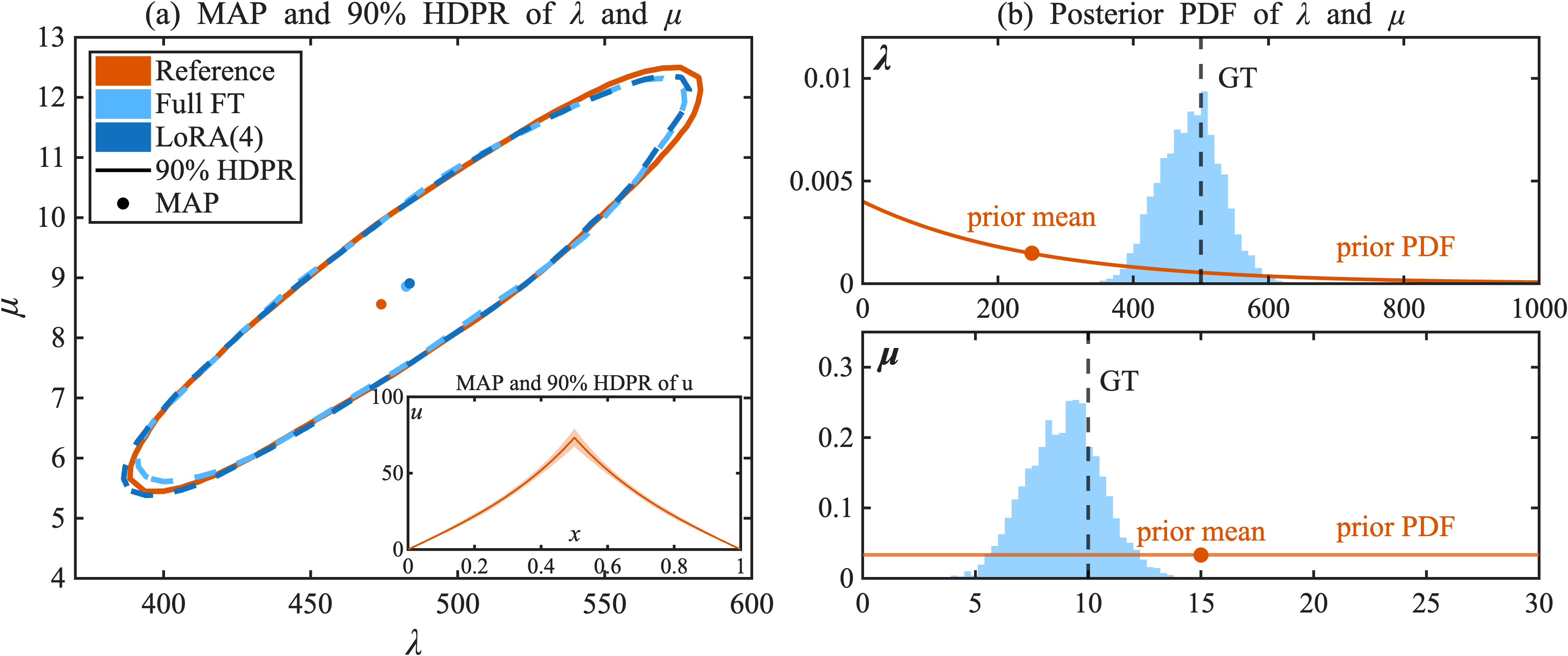}
  \caption{   
      Inference of gene expression parameters $\lambda$ (birth rate) and $\mu$ (degradation rate) using B-BiLO. 
      (a): The 90\% highest-density posterior region (HDPR) and maximum a posteriori (MAP) estimates of the solution obtained using B-BiLO with LoRA rank 4, Full FT, and the reference method (MH with numerical solver).
      Inset: 90\% HDPR and MAP of the solution $u(x)$.
      (b): Ground truth (GT), prior, and posterior distributions for $\lambda$ and $\mu$, with posterior sampled via B-BiLO using LoRA rank 4. 
      }
    \label{f:pp}
\end{figure}

\subsection{Example 4: Darcy Flow Problem}
\label{ss:darcy}
Our framework can be generalized to inferring unknown functions in the PDE, as described in Methods (Sec. \ref{method:learnfun}).
We consider the following 2D Darcy flow problem on the unit square $\Omega = [0,1]\times[0,1]$ with Dirichlet boundary conditions:
\begin{equation}
    -\nabla \cdot (D(\bx) \nabla u) = f(\bx) \quad \text{in } \Omega
\end{equation}
where $f(\bx) = 10$. 
The unknown spatially varying diffusion coefficient $D(\bx)$ is represented by a truncated 64-term Karhunen–Loève expansion with standard normal priors on the coefficients, modeling high- and low-diffusivity regions (3 to 12) across the domain (Details are provided in the SM) \cite{huangEfficientDerivativefreeBayesian2022}.

The reference result is obtained by MH sampling with numerical solutions computed on a 61$\times$61 grid. Noisy observations are collected at $21 \times 21$ evenly spaced grid points, and the residual loss is evaluated at $61 \times 61$ collocation points. The noise level is set to $\sigma_d = 0.01$, corresponding to approximately 10\% of the maximum value of the solution $u$.
Figure~\ref{f:darcy2d_sig} illustrates the posterior inference of the spatially varying diffusion coefficient $D(\bx)$ in the 2D Darcy flow problem. Panel (a) shows the ground truth (GT) diffusion coefficient. Panel (b) presents the point-wise posterior means of $D(\bx)$ obtained using B-BiLO with LoRA rank 4, which captures the key spatial features of the GT. Panel (c) quantifies the uncertainty in $D(x)$ along a fixed 1D slice of the domain, indicated by the dashed line in (a), showing the region of mean $\pm$ two standard deviations.
The relative L2 error of the B-BiLO posterior mean with respect to the reference is 6.2\%. Additional visualizations are provided in the SM.
\begin{figure}[!htb]
  \centering
  \includegraphics[keepaspectratio=true, width=\textwidth]{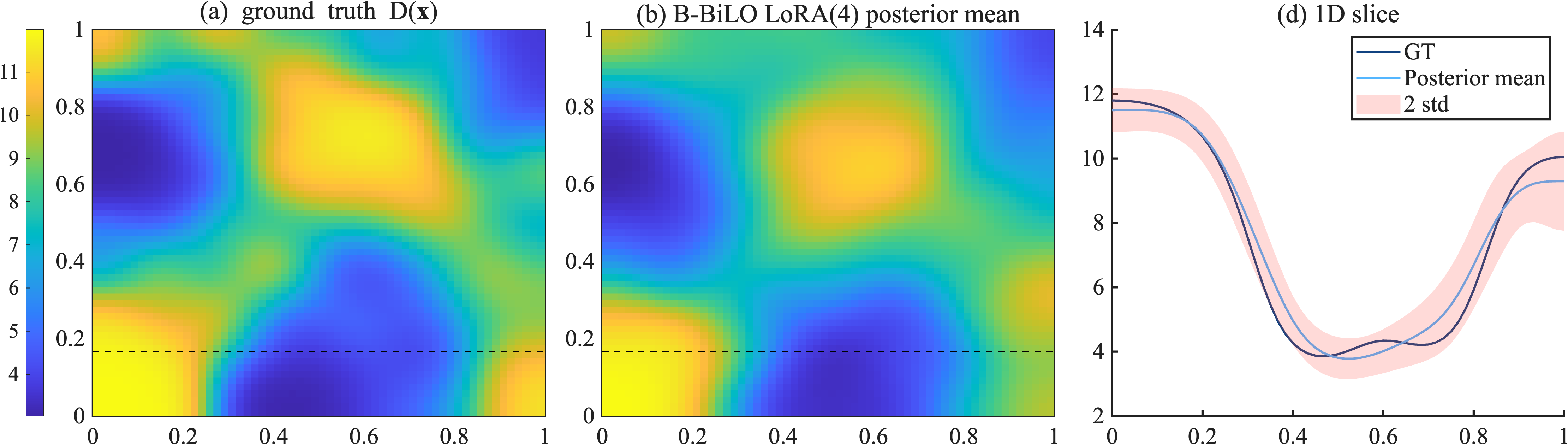}
  \caption{
Posterior inference of the spatially varying diffusion coefficient $D(\bx)$ in the 2D Darcy flow problem. 
(a) The ground truth $D(\bx)$;
(b) The posterior mean of $D(\bx)$ obtained using B-BiLO with LoRA rank 4;
(c) Mean $\pm$ 2 standard deviations of the posterior $D(\bx)$ evaluated along a 1D slice of the domain (dashed line in (a) and (b)). 
}
  \label{f:darcy2d_sig}
\end{figure}

\section{Discussion}
\label{discussion}

In this paper, we have introduced a novel framework (B-BiLO) for efficiently solving Bayesian PDE inverse problems.
Our approach enforces strong PDE constraints and integrates gradient-based MCMC methods, significantly enhancing computational efficiency and accuracy in parameter inference and uncertainty quantification. 
Through numerical experiments, we demonstrated that our method provides superior performance compared to Bayesian Physics-Informed Neural Networks (BPINNs).
Notably, our approach avoids the computational challenges of directly sampling high-dimensional Bayesian neural network weights and mitigates issues arising from ill-conditioned PDE residual terms.
Moreover, by incorporating LoRA, we substantially reduce computational overhead during the fine-tuning phase without compromising the accuracy or robustness of Bayesian inference.
Our results on various problems  highlight the effectiveness of our method.
A promising direction for future work is extending the B-BiLO framework to three-dimensional PDEs and inverse problems with high-dimensional parameter spaces, such as full-field identification problems in geophysics or medical imaging. Tackling these challenges will likely require integration with more efficient sampling algorithms, improved neural network architectures tailored to complex PDEs, and faster fine-tuning strategies to maintain scalability.

\section{Methods}
\label{method}


\subsection{Architecture}
\label{method:arch}
The local operator $u(\bx,\theta; W)$ is represented by a standard multi-layer perceptron (MLP). The network input is the concatenation of the coordinate and the PDE parameter, $h^{(0)} = [\bx;\theta] \in \mathbb{R}^{d+m}$. The hidden layers are defined by $h^{(l)} = \sigma(W^{(l)} h^{(l-1)} + b^{(l)})$ for $l = 1,\dots,L$, where $W^{(l)} \in \mathbb{R}^{p \times d_{l-1}}$ and $b^{(l)} \in \mathbb{R}^{p}$, with $d_{0} = d+m$ and $d_{l} = p$ for all hidden layers, and $\sigma$ is the activation function ($\tanh$ in this work).
The network output is $m(\bx,\theta; W) = W^{(L+1)} h^{(L)} + b^{(L+1)}$, where $W^{(L+1)} \in \mathbb{R}^{1 \times p}$ and $b^{(L+1)} \in \mathbb{R}$.
The collection of all trainable parameters is $W = \{ W^{(l)}, b^{(l)} \}_{l=0}^{L+1}$.
A final transformation can be applied to the output of the MLP to enforce the boundary condition or the initial condition \cite{dongMethodRepresentingPeriodic2021,sukumarExactImpositionBoundary2022}. 
For example, if the PDE is defined on a unit interval \([0, 1]\) with Dirichlet boundary conditions \(u(0) = u(1) = 0\), the BiLO solution can be represented as
$u(x, \theta; \wnn) = m(x,\theta;W) (1 - x) x$.
Alternatively, the boundary condition can be imposed by additional loss terms \cite{raissiPhysicsinformedNeuralNetworks2019}.
This MLP architecture is presented as a basic example and to facilitate the explanation of Low-Rank Adaptation in Section~\ref{method:lora}.
More advanced architectures can also be incorporated, including residual connections \cite{heDeepResidualLearning2015}, random Fourier features \cite{wangEigenvectorBiasFourier2021}, and other modified MLP designs \cite{wangUnderstandingMitigatingGradient2021}.

\subsection{Algorithm}
\label{method:algo}
The key component of our method is to compute the gradient of the potential energy $U(\theta)$.
In BiLO, the potential energy depends on the weights of the neural network and is denoted as $U(\theta, W)$.
At the optimal weights $W^*(\theta)$, the local operator $u(\bx, \theta; W^*(\theta))$, the potential energy $U(\theta, W^*(\theta))$, and its gradient $\nabla_\theta U(\theta, W^*(\theta))$ are expected to approximate the exact local operator $u(\bx, \theta)$, the true potential energy $U(\theta)$, and the true gradient $\nabla_\theta U(\theta)$, respectively.
Therefore, each time $\theta$ is updated, we must solve the lower level problem. The complete B-BiLO algorithm is detailed in Algorithm \ref{alg:bilohmc}.

In practice, we solve the lower level problem to some tolerance $\epsilon$ using our subroutine \texttt{LowerLevelIteration} in Algorithm \ref{alg:fullft}.
It is written as a simple gradient descent method, but in practice, we can use more advanced optimization methods such as Adam \cite{kingmaAdamMethodStochastic2017}.
In the HMC, we call the subroutine \texttt{LowerLevelIteration} each time before computing the gradient of the potential energy $\nabla_\theta U(\theta, W)$.

\begin{algorithm}
  \caption{B-BiLO \label{alg:bilohmc}}
  \begin{algorithmic}[1]
  \Require Initial states for $\theta^{(0)}$, time step size $\delta t$, and number of leapfrog steps $L$
  \For{$k = 1, 2, \dots, N$}
      \State Sample $r \sim \mathcal{N}(0, M)$
      \State $(\theta_0, r_0) \gets (\theta^{(k-1)}, r)$
      \State $H_0 \gets H(\theta_0, r_0; W)$
      \State \Call{LowerLevelIteration}{$\theta_{0}, W$}
      \For{$i = 0, 1, \dots, L-1$}
          \State $r_{i+1/2} \gets r_i - \frac{\delta t}{2} \nabla_\theta U(\theta_i, \wnn)$
          \State $\theta_{i+1} \gets \theta_i + \delta t M^{-1} r_{i+1/2}$
          \State \Call{LowerLevelIteration}{$\theta_{i+1},W$}
          \State $r_{i+1} \gets r_{i+1/2} - \frac{\delta t}{2} \nabla_\theta U(\theta_{i+1}, \wnn)$
      \EndFor
      \State $H_L \gets H(\theta_L, r_L; W)$
      \State $\alpha \gets \min\left(1, \exp \left(H_0 - H_L\right)\right)$
      \State with probability $\alpha$, set $\theta^{(k)} \gets \theta_L$; otherwise, set $\theta^{(k)} \gets \theta_0$
  \EndFor
  \end{algorithmic}
\end{algorithm}

\begin{algorithm}
  \caption{\label{alg:fullft}
LowerLevelIteration($\theta$, $W$)}
  \begin{algorithmic}[1]
    \While{$L_{\text{LO}}(\theta, \wnn) > \epsilon$}
      \State $\wnn \gets \wnn - \alpha_{\wnn} \nabla_{\wnn} L_{\text{LO}}(\theta, \wnn)$
    \EndWhile
  \end{algorithmic}
\end{algorithm}

\subsection{Theoretical analysis}
\label{method:theory}
The accuracy of the leapfrog integrator depends on the accuracy of the potential energy gradient, $\nabla_\theta U(\theta, \wnn)$, and the accuracy of the potential posterior distribution depends on the accuracy of the PDE solution.
Both errors are attributed to the inexact solution of the lower-level optimization problem.

The \textbf{ideal optimal weights} $W^*(\theta)$ satisfy the PDE for all $\theta$:
$ \cF[u(\cdot, \theta,W^*(\theta)), \theta]=\bzero$.
Its practical \textbf{approximation} is denoted by $\wbar$, which is obtained by terminating the optimization once the lower-level loss is within a specified tolerance, $\LLO(\theta, \wbar) \le \epsilon$. 
The true gradient of the upper-level objectives is $g_{true}(\theta) = d_\theta U(\theta, W^*(\theta))$.
In BiLO, we compute an approximate gradient $g_a(\theta) = \nabla_\theta U(\theta, \wbar)$, which does not consider the dependency of $\wbar$ on $\theta$.
Under mild assumptions of PDE operator stability, smoothness, and Lipschitz continuity of the solution map, we proved that $\norm{g_a-g_{true}} = O(\epsilon)$ \cite{zhangBiLOBilevelLocal2025a}.

In addition to the error of the approximate gradient, inexact lower-level optimization also introduces a static error in the posterior distribution targeted by the BiLO-HMC sampler. 
Specifically, denote $u^* = u(\cdot, \theta; W^*(\theta))$, and $\ubar = u(\cdot, \theta; \wbar)$.
The practical posterior $\pibar(\theta)$ is based on inexact solution of the PDE $\ubar$, while the ideal posterior $\pi_*(\theta)$ is based on the exact solution $u^*$. In Theorem 1, given in the SM, we quantify this static discrepancy by bounding the Kullback–Leibler divergence:
$$\KLdiv{\pibar}{\pi_*} = O(\epsilon).$$

\subsection{Efficient Sampling with Low-Rank Adaptation (LoRA)}
\label{method:lora}
The neural network weights $W$ are randomly initialized. Denote the initial weights as $W_{\rm rand}$.
Sampling the posterior distribution $P(\theta|D)$ requires an initial guess for the PDE parameters $\theta_0$.
During sampling, BiLO-HMC generates a sequence of PDE parameters $\theta^{(k)}_i$ for $k = 0, 1, 2, \dots, N$ and $i = 0, 1, \dots, L$, which is the $i$-th step of the leapfrog integrator for the $k$-th proposal. 
The lower-level problem needs to be solved to obtain the neural network weights $\wnn(\theta^{(k)}_i)$. 
We separate the training process into two stages: pre-training and fine-tuning. And discuss how to speed up the training process.

\paragraph{Pre-training (Initialization)}
During pre-training, the neural network weights are initialized randomly, and we fix the PDE parameter $\theta^{(0)}$ and minimize the local operator loss $\LLO(\theta^{(0)}, \wnn)$ to obtain the initial weights $\wnn^*(\theta^{(0)})$. 
The pre-training process can be sped up if the numerical solution of the PDE $u$ at $\theta_0$ is available and included as an additional data loss.
\begin{equation}
  \wnn^* = \arg\min_{\wnn}\left( \LLO(\theta_0,\wnn) + \| u(\cdot, \theta_0; \wnn) - u(\cdot, \theta_0) \|^2\right)
\end{equation}
The additional data loss is not mandatory for training the local operator with fixed $\theta_0$. 
However, the additional data loss can speed up the training process \cite{krishnapriyanCharacterizingPossibleFailure2021}, and is computationally inexpensive as we only need one numerical solution corresponding to $\theta_0$.

\paragraph{Fine-tuning (Sampling)} In the \textit{fine-tuning} stage, we train the local operator for the sequence $\theta^{(k)}_i$ for $k = 0, 1, 2, \dots, N$ and $i = 0, 1, \dots, L$, where the changes in the PDE parameters are small.
We do not train from scratch. Instead, we fine-tune the neural network weights $\wnn$ based on the previous weights, which usually converges much faster.


In Algorithm \ref{alg:fullft}, we fine-tune all the weights of the neural network $\wnn$. We call this approach \textbf{full fine-tuning (Full FT)}.
For large neural networks, this can still be computationally expensive. 
In the next section, we describe how to use the Low-Rank Adaptation (LoRA) technique \cite{huLoRALowRankAdaptation2021} to reduce the number of trainable parameters and speed up the fine-tuning process.

\paragraph{Efficient Fine-tuning with Low Rank Adaptation (LoRA)}
Low-Rank Adaptation (LoRA) is a technique initially developed for large language models \cite{huLoRALowRankAdaptation2021}.
In LoRA, instead of updating the full weight matrix, we learn an low-rank modification to the weight matrix. Below we describe the details of LoRA in the context of BiLO, following the notations used in Sec. \ref{method:arch}.

Suppose the initial guess of the PDE parameter is $\theta^{(0)}$. After pre-training, let ${\mathbf{W}}^{(l)}_0 \in \mathbb{R}^{p \times p}$ be the weight matrix of the $l$-th layer of the neural network after pre-training, and $W_0$ denote the collection of all weight matrices of the neural network after pre-training at fixed $\theta^{(0)}$.

During sampling, after each step of the leapfrog integrator, we arrive at some different PDE parameters $\theta$. We need to solve the lower level problem again to obtain the new weight matrix $ {\bW^{(l)}}(\theta)$.
Consider the effective update of the weight matrix
$$\Delta \bW^{(l)}(\theta) = \bW^{(l)}(\theta) - \bW^{(l)}_0.$$
In LoRA, it is assumed that the change in weights has a low-rank structure, which can be expressed as a product of two low-rank matrices: 
$$\Delta \bW^{(l)} = \mathbf{A}^{(l)} \mathbf{B}^{(l)},$$ 
where $\mathbf{A}^{(l)} \in \mathbb{R}^{p \times r}$ and $\mathbf{B}^{(l)} \in \mathbb{R}^{r \times p}$ are low-rank matrices with $r \ll p$. 

Denote the collection of low-rank matrices as $W_{LoRA} = \{\mathbf{A}^{(l)}, \mathbf{B}^{(l)}\}_{l=1}^L$. 
Then neural network is represented as
$u(\bx, \theta; W_0 + W_{LoRA})$, where $W_0 + W_{LoRA}$ means that the $l$-th layer weight matrix is given by $\bW^{(l)} = \bW^{(l)}_0 + \mathbf{A}^{(l)} \mathbf{B}^{(l)}$.
This approach aims to reduce the complexity of the model while preserving its ability to learn or adapt effectively.
With LoRA, our lower level problem becomes
\begin{equation}
  W_{LoRA}^*(\theta) = \arg\min_{W_{LoRA}} \LLO(\theta, W_0 + W_{LoRA})
\end{equation}
where $\theta$ are the PDE parameters computed in leapfrog. 
The full FT subroutine (Algorithm \ref{alg:fullft}) can be replaced by the LoRA fine-tuning in Algorithm \ref{alg:loraft}. 
In this work, we use the original LoRA \cite{huLoRALowRankAdaptation2021} as a proof-of-concept and we note that more recent developments of LoRA might further increase efficiency and performance \cite{zhangAdaLoRAAdaptiveBudget2023,wangLoRAGALowRankAdaptation2024,hayouLoRAEfficientLow2024}.

\begin{algorithm}
  \caption{\label{alg:loraft}
LowerLevelIteration-LoRA($\theta$) }
  \begin{algorithmic}[1]
    \Require Initial weights $W_0$ after pre-training, rank $r$, learning rate $\alpha_{W_{LoRA}}$, and stopping tolerance $\epsilon$
    \While{$L_{\text{LO}}(\theta, W_0 + W_{LoRA}) > \epsilon$}
      \State $W_{LoRA} \gets W_{LoRA} - \alpha_{W_{LoRA}} \nabla_{W_{LoRA}} \LLO(\theta, W_0 + W_{LoRA})$
    \EndWhile
    \State \Return $W_0 + W_{LoRA}$
  \end{algorithmic}
\end{algorithm}

\paragraph{Trade-off of LoRA} 
The number of trainable weights $W_{LoRA}$ is much smaller than the full weight matrix $W$, and therefore LoRA reduces the memory requirements significantly at each iteration. However, evaluating the neural network (forward pass) and computing the gradient via back propagation still requires the full weight matrix $W = W_0 + W_{LoRA}$. Using the additional low-rank matrices $\mathbf{A}^{(l)}$ and $\mathbf{B}^{(l)}$ might incur additional computational overhead. Thus, LoRA is more beneficial for large neural networks (in our case, large $p$).

In addition, with LoRA, the optimization problem is solved in a lower dimensional subspace of the full weight space. Therefore, it is expected that more iterations are needed to reach the same level of loss as full-model fine-tuning \cite{chenCELoRAComputationEfficientLoRA2025}. 
On the other hand, optimizing over a subspace can serve as a regularization \cite{bidermanLoRALearnsLess2024}.
In numerical experiments, we show that as the model size increases, the benefit outweighs the cost, and LoRA can significantly speed up the fine-tuning process.

\subsection{Inferring Unknown Functions}
\label{method:learnfun}
We describe in general how to infer unknown functions in the BiLO framework \cite{zhangBiLOBilevelLocal2025a}.
Suppose the PDE depends on some unknown functions $f(\bx)$, such as the spatially varying diffusivity in a diffusion equation. Denote the PDE as
\begin{equation}
  F(D^ku(\bx),...,D u(\bx), u(\bx), f(\bx)) = \bzero
\end{equation}
We introduce an auxiliary variable $z = f(\bx)$, and we aim to find a local operator $u(\bx, z)$ such that $u(\bx, f(\bx))$ solves the PDE locally at $f$. 
We introduce the augmented residual function, which has an auxiliary variable $z$:
\begin{equation}
  a(\bx, z) := F(D^ku(\bx,z),...,D u(\bx,z), u(\bx,z), z).
\end{equation}
And the conditions for a local operator are:
\begin{enumerate}[leftmargin=*, itemsep=0pt]
  \item $a(\bx, f(\bx)) = \bzero$
  \item $\grad_{z} a(\bx, z)|_{z=f(\bx)} = d_z \cF(D^ku(\bx,z),...,D u(\bx,z), u(\bx,z), z)|_{z=f(\bx)} = \bzero$.
\end{enumerate}
Condition 1 states that the function $u(\bx, f(\bx))$ has zero residual, and condition 2 means that small perturbations of $f(\bx)$ result in small changes in the residual,
e.g., the PDE is approximated to second order in the perturbation size (see also \cite{zhangBiLOBilevelLocal2025a}).

The unknown function $f$ can be represented by a Bayesian neural network  \cite{nealBayesianLearningNeural1996,yangBPINNsBayesianPhysicsinformed2021} or a Karhunen-Loève (KL) expansion \cite{rasmussenGaussianProcessesMachine, huangEfficientDerivativefreeBayesian2022,liFourierNeuralOperator2024}. We denote the parameterized function by $f(\bx,\theta)$, where $\theta$ are the weights of the Bayesian neural network or the coefficients of the KL expansion. The Local Operator loss in this case is
\begin{equation}
  \LLO(\theta, \wnn) = \norm{a(\cdot,f(\cdot,\theta),\wnn)}^2_2 + \wdr \norm{\nabla_{z} a(\cdot, f(\cdot,\theta), \wnn)}^2_2,
  \label{eq:llo_fun}
\end{equation}
and the same bi-level Bayesian inference problem \eqref{eq:bi_bayes} is solved.

\backmatter

\bmhead{Code Availability}
The code for the numerical experiments is available at \url{https://github.com/Rayzhangzirui/BiLO}.

\bmhead{Acknowledgments}
R.Z.Z and J.S.L thank Babak Shahbaba for the GPU resources. 
R.Z.Z. thanks the NVIDIA Academic Grant Program for the NVIDIA RTX PRO 6000 Blackwell GPU.
J.S.L acknowledges partial support from the National Science Foundation through grants DMS-2309800, DMS-1953410 and DMS-1763272 and the Simons Foundation (594598QN) for an NSF-Simons Center for Multiscale Cell Fate Research. 
C.E.M was partially supported by a NSF CAREER grant DMS-2339241.



\bibliography{bilolora}

\end{document}


\maketitle

\begin{appendices}

\section{BiLO for PDE-Constrained Optimization}
\label{ap:part1}
For completeness, we briefly review the BiLO framework for solving PDE-constrained optimization problems developed in \cite{zhangBiLOBilevelLocal2025a}.
We consider the following PDE-constrained optimization problem:
\begin{equation}
  \begin{aligned}
    &\min_{\theta} \quad \lVert u -  \hu\lVert^2_2 \\
    &\textrm{s.t.} \quad F(\cD^ku(\bx),...,\cD u(\bx), u(\bx), \theta) = \mathbf{0}\\
  \end{aligned}
  \label{eq:opt_scalar}
\end{equation}
where $\hat{u}$ is the observed data.

Using the same definition of the local operator and local operator loss as in the main text, we solve the following bi-level optimization problem:
\begin{equation}
  \label{eq:bilo_scalar}
  \begin{cases}
    \theta^* = \arg\min_{\theta} \ldat(\theta, \wnn^*(\theta)) \\
    \wnn^*(\theta) = \arg\min_{\wnn} \lopt(\theta, \wnn) \\
  \end{cases}
\end{equation}
In the upper level problem, we find the optimal PDE parameters $\theta$ by minimizing the data loss 
$ \ldat(\theta, \wnn) = \lVert u(\cdot, \theta;W) -  \hu\lVert^2_2$
with respect to $\theta$.
In the lower level problem, we train a network to approximate the local operator $u(\bx, \theta; \wnn)$ by minimizing the local operator loss with respect to the weights of the neural network.

The bi-level optimization problem \eqref{eq:bilo_scalar} can be solved by simultaneous gradient descent at both levels:
\begin{equation*}
      \label{eq:bilogd}
      \begin{cases}
        \theta^{k+1} = \theta^{k} - \alpha_{\theta} \grad_{\theta} \ldat(\theta^{k}, \wnn^{k})\\
        \wnn^{k+1} = \wnn^{k} - \alpha_{\wnn} \grad_{\wnn} \lopt(\theta^{k}, \wnn^{k})
      \end{cases}
\end{equation*}
Theoretically, we show that when the lower-level problem is solved exactly, the gradient of the upper-level loss is exact. When the lower-level problem is solved to a tolerance $\epsilon$, the error between the approximate and exact upper-level gradient is of order $\epsilon$ (see Theorems 1 and 2 in \cite{zhangBiLOBilevelLocal2025a}).

In \cite{zhangBiLOBilevelLocal2025a}, the effectiveness of the method is demonstrated on a variety of PDEs, including inferring the diffusion and proliferation rate of Fisher-KPP equation, inferring stochastic rate from particle data (elliptic PDE with singular forcing), inferring the initial condition of a heat equation and inviscid Burgers' equation, inferring the spatially varying diffusion coefficient in a 2D Darcy flow problem, and a Glioblastoma inverse problem using patient data.
The method is compared to PINN, Neural Operator, and adjoint methods. Overall, we showed that BiLO is robust to sparse and noisy data, and eliminates the need to balance the residual and the data loss \cite{zhangBiLOBilevelLocal2025a}. 

\section{Theoretical Analysis}

We first briefly review the theoretical analysis of BiLO on the error of the upper-level gradient, introduced by the inexact minimization of the lower-level problem, which is presented in \cite{zhangBiLOBilevelLocal2025a}.

\paragraph{Setup}
Consider a bounded domain $\Omega \subset \mathbb{R}^d$. 
The PDE parameter $\theta \in \Theta \subset \mathbb{R}^m$.
The weights $W \in \mathbb{R}^n$.
The PDE operator $\cF: H^1_0(\Omega) \times \Theta \to H^{-1}(\Omega)$.
The PDE solution map (parameterized by weights $W$)
$u: \Theta \times \mathbb{R}^n \to H^1_0(\Omega)$. 
The potential energy has the form $\cU[u, \theta] = \ell[u] -\log P(\theta)$, where $\ell[u]$ is the data-fit term and is a functional on the solution of the PDE, and $P(\theta)$ is the prior distribution of $\theta$.
For simplicity, we also sometimes write $U(\theta, W)= U[u(\theta, W),\theta]$.
The residual as a function of $\theta$ and $W$ is defined as
\begin{equation}
  r(\theta,W) := \cF(u(\theta,W), \theta)
\end{equation}
Denote the (partial) Fréchet derivative of $\cF$ by $\cF_u$ and $\cF_\theta$. 
The residual-gradient is given by
\begin{equation}
  \nabla_\theta r(\theta,W):= \cF_u(u(\theta,W), \theta)(\gradW u(\theta,W)) + \cF_\theta(u(\theta,W), \theta)
\end{equation}

The \textbf{ideal optimal weights} $W^*(\theta)$ satisfies the PDE for all $\theta$:
$$ \cF(u(\theta,W^*(\theta)), \theta)=0.$$
Its practical \textbf{approximation} is denoted by $\wbar$, which is obtained by terminating the optimization once the lower-level loss is within a specified tolerance, 
$$\LLO(\theta, \wbar) \le \epsilon.$$ 
We also denote $u^* = u(\theta, W^*)$ the solution at the ideal optimal weights $W^*$, $\bar{u} = u(\theta, \wbar)$ the solution at the approximate weights $\wbar$.
We list the assumptions for the theoretical analysis, which are similar to those in  \cite{zhangBiLOBilevelLocal2025a}.
\begin{assumption}[Assumptions for Hypergradient Analysis] \label{assump:unified_new}
Consider a parameterized PDE $\mathcal{F}(u, \theta) = 0$ on a bounded domain $\Omega$ with $\theta \in \mathbb{R}^m$.
\begin{itemize}
    \item[(i)] \textbf{Inexact Minimization:} The lower-level optimization for the weights $\wbar$ terminates when the total local operator loss is within a tolerance $\epsilon$:
    $$\LLO(\theta, \wbar) = \norm{r(\theta, \wbar)}_{L^2}^2 + w_{\rm rgrad}\norm{\grad_{\theta} r(\theta, \wbar)}_{H^{-1}}^2 \le \epsilon.$$
    Since $\wdr$ is some fixed weight, without loss of generality, we can assume that both the residual loss and the residual-gradient loss are controlled by $\epsilon$.
    \item[(ii)] \textbf{PDE Operator Properties:} The operator $F(u, \theta)$ are sufficiently Fréchet differentiable and stable, that is, if $\cF(u,\theta) = 0$ and $\norm{\cF(v, \theta)}\leq \varepsilon$, then $\norm{v-u}\leq C \varepsilon$ for some constant $C$. 
    The linearized operator at the $u$, denoted $\bL_{u}[\cdot] := F_u(u^*, \theta)[\cdot]$, is stable, that is, if
    $\bL_{u}[v] = f$, then $\norm{v} \leq C \norm{f}$ for some constant $C$.
    \item[(iii)] \textbf{Smoothness:} The data fitting term $\ell$ is Lipschitz continuous. The solution $u$ is Lipschitz continuous in the weights $W$ and the parameters $\theta$, and has bounded derivatives with respect to $\theta$.
\end{itemize}
\end{assumption}

The approximate gradient of the potential energy is given by
\[g_a(\theta) = \nabla_\theta U(\theta,\wbar)\] 
And the true hypergradient is
\[g_{true}(\theta) = d_\theta U(\theta,W^{*}(\theta))\] 
Since the prior $P(\theta)$ is independent of $W$, this difference arises solely from the data-fit term $\ell$.
As a direct consequence of Theorem 2 in \cite{zhangBiLOBilevelLocal2025a}, we have the following result:
\[
\norm{g_a(\theta) - g_{true}(\theta)}= O(\epsilon)
\]


The preceding results bounds the dynamic error in the HMC sampler's gradient.
A separate and more fundamental issue is the static error in the sampler's target distribution.
Because the lower-level problem is solved inexactly, the algorithm targets an approximate posterior $\pibar$, which is computed using the approximate solution $\ubar$, instead of the ideal one $\pi^*$, which requires $u^*$.
The total error in the BiLO-HMC method thus has two distinct components: the dynamic sampler error (order $O(\epsilon)$) and this static target error.
The following theorem isolates and bounds this static error.

\begin{theorem}[Posterior Perturbation Bound]
\label{thm:posterior_perturbation}
Under Assumption \ref{assump:unified_new}, the KL divergence between the practical target posterior $\pibar(\theta) \propto \exp(-U(\theta, \wbar(\theta)))$ and the ideal posterior $\pi^*(\theta) \propto \exp(-U(\theta, W^*(\theta)))$ is order $O(\epsilon)$:
$$\KLdiv{\pibar}{\pi^*} = O(\epsilon)$$
\end{theorem}

\begin{proof}
\textbf{Step 1: Bounding the potential difference.} Let 
$$\Delta U(\theta) = U(\theta, \wbar(\theta)) - U(\theta, W^*(\theta)).$$
Since the prior $P(\theta)$ is independent of $W$, this difference arises solely from the data-fit term $\ell$.
Using the Lipschitz properties of $\ell$ and the stability of the PDE,
\begin{align*}
    |\Delta U|_\infty = \sup_{\theta\in\Theta} |\Delta U(\theta)| &\leq K \norm{\ubar-u^*} = O(\epsilon)
\end{align*}
\textbf{Step 2: Bounding the KL divergence.} The KL divergence is defined as
$$\KLdiv{\pibar}{\pi^*} = \mathbb{E}_{\pibar}[ \log(\pibar/\pi^*) ] = \log(Z^*/\bar{Z}) - \mathbb{E}_{\pibar}[\Delta U],$$
where $Z^*$ and $\bar{Z}$ denote the normalizing constants of $\pi^*$ and $\bar{\pi}$, respectively.
The resulting log-ratio is bounded by 
$$|\log(Z^*/\bar{Z})|
= -|\log(\mathbb{E}_{\pi^*}[\exp(-\Delta U)])| \le |\Delta U|_\infty,$$ which follows from Jensen's inequality.
Therefore:
\begin{align*} \KLdiv{\pibar}{\pi^*} &\le |\log(Z^*/\bar{Z})| + |\mathbb{E}_{\pibar}[\Delta U]|
\\& \le 2|\Delta U|_\infty \\
& = O(\epsilon)
\end{align*}
\end{proof}

\section{Difference from BPINN}
\label{method:bpinn}
We briefly review the Bayesian PINN (BPINN) framework \cite{yangBPINNsBayesianPhysicsinformed2021} and highlight the differences with BiLO.

\paragraph{Review of BPINN}
Within the BPINN framework, the solution of the PDE is represented by a Bayesian neural network $\upinn(\bx, W)$, 
where $W$ denotes the weights of the neural network that are taken to be random variables. 
Note that $\theta$ is not an input to the neural network.
For the Bayesian Neural Network, the prior distribution of the weights, $P(W)$, is usually assumed to be an i.i.d. Normal distribution.

The likelihood of the observation data $D_u = \{(\bx_u^{(i)}, \hu^{(i)})\}_{i=1}^{N_u}$, where $\hu^{(i)} = u(\bx_u^{(i)}) + \eta_d$, $\eta_d \sim \mathcal{N}(0, \sigma_d^2)$, is given by
\begin{equation}
  P(D_u|W) = \prod_{i=1}^{N_u} \frac{1}{\sqrt{2\pi \sigma_d^2}} \exp\left( -\frac{1}{2\sigma_d^2} \left| \upinn(\bx_u^{(i)},W) - \hu^{(i)}\right|^2 \right)
\end{equation}
Notice that the likelihood of the observation data $D_u$ only depends on the neural network weights $W$, and does not depend on the PDE parameters $\theta$ as BiLO does.

In BPINN, we denote the PDE as $\mathcal{G}(u(\bx),\theta) = f(\bx)$, where $\mathcal{G}$ is the differential operator and $f$ is the forcing term.
Given noisy measurements of the forcing term  $D_f = \{(\bx_f^{(i)}, \hat{f}^{(i)})\}_{i=1}^{N_f}$, where
$\hat{f}^{(i)} = f(\bx_f^{(i)}) + \eta_f$ and $\eta_f \sim \mathcal{N}(0, \sigma_f^2)$, then
\begin{equation}
  P(D_f|W, \theta) = \prod_{i=1}^{N_f} \frac{1}{\sqrt{2\pi \sigma_f^2}} \exp\left( -\frac{1}{2\sigma_f^2} \left| \mathcal{G}(\upinn(\bx_f^{(i)}; W), \theta) -\hat{f}^{(i)}\right|^2 \right)
  \label{eq:residual}
\end{equation}
Assuming $\theta$ and $W$ are independent, the following joint posterior is sampled using HMC.
\begin{equation}
  P(W, \theta|D_u, D_f) \propto P(D_u|W) P(D_f|W, \theta) P(W) P(\theta).
\end{equation}

\paragraph{Challenges for BPINNs}
One of the main differences between BPINN and BiLO lies in the treatment of the neural network weights, $\wnn$. 
For PDE inverse problems, the PDE parameters $\theta$ are usually low-dimensional. 
In the BPINN framework, the weights $\wnn$ are treated as random variables that must be sampled from the joint posterior distribution alongside the PDE parameters $\theta$.
However, $\wnn$ is high-dimensional, making sampling challenging.

In contrast, the BiLO framework treats the weights $\wnn$ as deterministic variables. 
For any given parameter $\theta$ in the sampling process, the optimal weights $\wnn^*(\theta)$ are found through a direct optimization of the lower-level problem, as defined in the second equation in Eq. \eqref{eq:bilo_scalar}. 
This approach avoids placing a prior on $\wnn$ and more importantly bypasses the challenge of sampling from the high-dimensional and often complex posterior distribution of the network weights.

Modeling uncertainty in the forcing term $f$ can also be nuanced.
In a BPINN, the solution $u$ is represented by a Bayesian neural network, and uncertainty originates from the prior distribution placed on the network weights, $P(W)$. Uncertainty over the weights $W$ then propagates through the differential operator $\mathcal{G}$ to induce a distribution on the model's estimate of the forcing term, $f=\mathcal{G}(u(\cdot; W), \theta)$. This ``top-down'' uncertainty propagation from the model's prior contrasts with cases where it is more physically meaningful for uncertainty in the forcing term $f$ itself (e.g., from noisy measurements or an explicit prior, $P(f)$) to propagate to the solution $u$.

Another nuance lies in the interpretation of $\sigma_f$ and $D_f$.
When $D_f$ represents noisy physical measurements of the forcing term $f$, 
and $\sigma_f$ is the noise level,
the likelihood $P(D_f|W, \theta)$ allows one to incorporate the PDE.
However, if no such measurements are available, then the physics-informed component of the likelihood vanishes.
On the other hand, if there is no noise in $D_f$, then $\sigma_f$ is 0, and the likelihood becomes singular and cannot be sampled by most sampling methods.
An alternative interpretation is to view $D_f$ and $\sigma_f$ as user-defined constructs. Similar to the residual loss in PINNs, they serve as a soft PDE constraint.
In this case, $\sigma_f$ plays the role of a penalty parameter: smaller values lead to more accurate solution of the PDE.

Regardless of the interpretation, having a small $\sigma_f$ is computationally demanding.
This challenge stems from the stability requirements of HMC, which dictates that the leapfrog step size, $\delta t$, is limited by the inverse square root of the potential energy's maximum curvature~\cite{nealMCMCUsingHamiltonian2011}. In the BPINN framework, this maximum curvature is proportional to $\lambda_{\rm max}/\sigma_f^2$, where $\lambda_{\rm max}$ is the largest eigenvalue of the Hessian of the unscaled residual loss, $\lres$. This unscaled loss is often ill-conditioned, with $\lambda_{\rm max}$ values known to exceed $10^3$~\cite{rathoreChallengesTrainingPINNs2024}. Consequently, applying the HMC stability limit imposes a scaling law on the step size, forcing $\delta t$ to be of order $\cO(\sigma_f/\sqrt{\lambda_{\rm max}})$.
In contrast, the BiLO framework does not require sampling the weights $W$, and thus avoids the stability limit imposed by the Hessian of the residual loss. This allows BiLO to use larger step sizes, leading more distant proposals and more efficient sampling.

\section{Review of Markov Chain Monte Carlo (MCMC) Methods}
\label{ss:hmc}

In this section, we provide a quick overview of the sampling methods that appeared in this work, including the Metropolis-Hastings (MH) algorithm and the Hamiltonian Monte Carlo (HMC). These methods are Markov Chain Monte Carlo (MCMC) methods that generate samples from a target distribution by constructing a Markov chain whose stationary distribution is the target distribution.

We present these methods independent of the BiLO framework. 
This is assuming that we have the parameter-to-solution map $\theta \mapsto u(\cdot, \theta)$, and we can compute the potential energy $U(\theta)$ and the gradient $\nabla_\theta U(\theta)$ accurately. With this assumption, in this section, we temporarily drop the dependence of the potential energy $U$ on the neural network weights $\wnn$. 

\paragraph{Metropolis-Hastings (MH) Algorithm} 
The Metropolis-Hastings (MH) algorithm \cite{hastingsMonteCarloSampling1970} is a classical and simple MCMC method. It requires a proposal distribution $Q(\theta'|\theta)$ to propose new samples $\theta'$ given the current sample $\theta$.
The proposal is then accepted or rejected based on the changes in the potential energy.
Algorithm \ref{alg:mh}, given below, summarizes the MH algorithm for sampling from the potential energy $U(\theta)$.
For simple proposal distributions, such as Gaussian or uniform distributions, MH is easy to implement and does not require computing the gradient of the potential energy. However, it can be inefficient as the acceptance rate may be low \cite{nealMCMCUsingHamiltonian2011}.
In this work, the MH algorithm is coupled with an accurate numerical PDE solver, and serves primarily as a reference method to assess the accuracy of alternative sampling approaches.

\begin{algorithm}
  \caption{Metropolis-Hastings (MH) Algorithm \label{alg:mh}}
  \begin{algorithmic}[1]
  \Require Initial state $\theta^{(0)}$, proposal distribution $Q(\theta'|\theta)$, number of iterations $N$
  \For{$k = 1, 2, \dots, N$}
      \State Sample $\theta' \sim Q(\cdot|\theta^{(k-1)})$
      \State Compute acceptance probability
      \begin{equation*}
        \alpha = \min\left(1, \exp\left[ U(\theta^{(k-1)}) - U(\theta') \right] \cdot \frac{Q(\theta^{(k-1)}|\theta')}{Q(\theta'|\theta^{(k-1)})} \right)
      \end{equation*}
      \State Sample $u \sim U[0, 1]$
      \If{$u < \alpha$}
          \State Accept the proposal: $\theta^{(k)} \gets \theta'$
      \Else
          \State Reject the proposal: $\theta^{(k)} \gets \theta^{(k-1)}$
      \EndIf
  \EndFor
  \end{algorithmic}
\end{algorithm}

\paragraph{Hamiltonian Monte Carlo (HMC)} 
In HMC, an auxiliary momentum variable $\rho$ is introduced, and the Hamiltonian is defined as
\begin{equation}
  H(\theta, \rho) = U(\theta) + K(\rho) = -\log P(\theta|D) + \frac{1}{2} \rho^T M^{-1} \rho,
\end{equation}
where $M$ is the mass matrix, which is symmetric and positive definite. $M$ can be user-defined or adaptively learned in a warmup phase.
HMC samples from the joint distribution $P(\theta, \rho) \propto \exp(-H(\theta, \rho))$ from the Hamiltonian dynamics.
\begin{equation}
    \frac{d\theta}{dt} = M^{-1} \rho, \quad
    \frac{d\rho}{dt} = -\nabla_\theta U(\theta).
\end{equation}
While numerous HMC variants exist, our work employs a standard implementation using the leapfrog integrator with a Metropolis-Hastings correction step.
This can be viewed as an instance of the MH algorithm, where the proposal distribution is defined by the Hamiltonian dynamics. The scheme 
requires computing the gradient of the potential energy.
Using the Hamiltonian dynamics, we can generate distant proposals with high acceptance rates. 
This property helps mitigate the random-walk behavior common in simpler MCMC methods, making HMC a particularly effective and widely-used sampler for exploring complex target distributions \cite{nealMCMCUsingHamiltonian2011}.

\paragraph{Leapfrog Integrator} The Hamiltonian dynamics can be solved using numerical integrators. The most commonly used integrator is the leapfrog method, which is a symplectic integrator that preserves the Hamiltonian structure \cite{nealMCMCUsingHamiltonian2011}. 
We sample a momentum variable $\rho$ from a Gaussian distribution $\mathcal{N}(0, M)$, and then simulate the Hamiltonian dynamics for $L$ steps with a fixed time step size $\delta t$.
For $i = 0, 1, \dots, L-1$:
\begin{equation}
  \begin{aligned}
    \rho_{i+\frac{1}{2}} &= \rho_i - \frac{\delta t}{2} \nabla_\theta U(\theta_i)\\
    \theta_{i+1} &= \theta_i + \delta t M^{-1} \rho_{i+\frac{1}{2}}\\
    \rho_{i+1} &= \rho_{i+\frac{1}{2}} - \frac{\delta t}{2} \nabla_\theta U(\theta_{i+1}),
  \end{aligned}
\end{equation}
Both $\delta t$ and $L$ are hyperparameters that are predefined or adaptively tuned during a warm-up phase.

\paragraph{Metropolis-Hastings Step}
Solving the Hamiltonian dynamics numerically introduces error, which can be corrected using the Metropolis-Hastings step \cite{hastingsMonteCarloSampling1970}. 
Starting from some state $(\theta_0, \rho_0)$, we perform $L$ leapfrog steps to arrive at $(\theta_L, \rho_L)$. 
We then perform a Metropolis-Hastings step: we accept the proposal $\theta_L$ with probability $\alpha$, where
$\alpha = \min\left(1, \exp\left(H(\theta_0, \rho_0) - H(\theta_L, \rho_L)\right)\right)$
If the proposal is accepted, we set $\theta^{(k)} = \theta_L$; otherwise, we set $\theta^{(k)} = \theta^{(k-1)}$. 
The momentum variable $\rho$ is discarded, and is resampled at the start of the next leapfrog step.

\paragraph{HMC Algorithm} 
We summarize the HMC algorithm in Algorithm \ref{alg:hmc}. $\theta^{(k)}$ denotes the $k$-th sample, while $\theta^{(k)}_i$ denotes the PDE parameters at the $i$-th step of the leapfrog integrator for the $k$-th proposal.
\begin{algorithm}
  \caption{Hamiltonian Monte-Carlo with Leapfrog Integrator \label{alg:hmc}}
  \begin{algorithmic}[1]
  \Require Initial states for $\theta^{(0)}$ and time step size $\delta t$
  \For{$k = 0, 1, 2, \dots, N$}
      \State Sample $r \sim \mathcal{N}(0, M)$
      \State $(\theta_0, r_0) \gets (\theta^{(k)}, r)$
      \For{$i = 0, 1, \dots, L-1$}
          \State $r_{i+1/2} \gets r_i - \frac{\delta t}{2} \nabla_\theta U(\theta_i)$
          \State $\theta_{i+1} \gets \theta_i + \delta t M^{-1} r_{i+1/2}$
          \State $r_{i+1} \gets r_{i+1/2} - \frac{\delta t}{2} \nabla_\theta U(\theta_{i+1})$
      \EndFor
      \State $\alpha \gets \min\left(1, \exp\left(H(\theta_0, r_0) - H(\theta_L, r_L)\right)\right)$
      \State With probability $\alpha$, set $\theta^{(k)} \gets \theta^{(k)}_L$; otherwise, set $\theta^{(k)} \gets \theta^{(k-1)}$
  \EndFor
  \end{algorithmic}
\end{algorithm} 
We note that sampling remains an active area of research. For example, 
many hyperparameters in HMC can be made adaptive. 
Further, there have been many advancements over HMC, such as the No-U-Turn Sampler (NUTS) \cite{hoffmanNoUTurnSamplerAdaptively}, higher order integrator \cite{hernandez-sanchezHigherOrderHamiltonian2021}, or stochastic HMC \cite{chenStochasticGradientHamiltonian2014} and we defer the use of such methods for future work.
For simplicity and for comparison with \cite{yangBPINNsBayesianPhysicsinformed2021},
we take $M$ as the identity matrix, and we use the leapfrog method with fixed step size $\delta t$ and length $L$, which is also one of the most common versions of HMC.

\section{Additional Results}
\label{additional results}

\subsection{Example 1: Comparison with BPINN}
\label{ap:ss:nonlinpoi}

For both B-BILO and BPINN, we use 4-layer fully connected neural network with 128 neurons and $\tanh$ activations. The model is trained using the Adam optimizer with AMSGrad, with a learning rate of $10^{-3}$ for the network weights. For HMC, we fix $L=10$ and cross validate $\delta t$. Residual loss is computed on 51 collocation points, and the noisy data is observed at 6 evenly spaced points in the domain.

Fig~\ref{f:nonlin_examples} shows additional inference results for the nonlinear Poisson problem using B-BiLO and BPINN with different $\sigma_f$.
\begin{figure}[htbp]
  \centering
  \includegraphics[keepaspectratio=true, width=\textwidth]{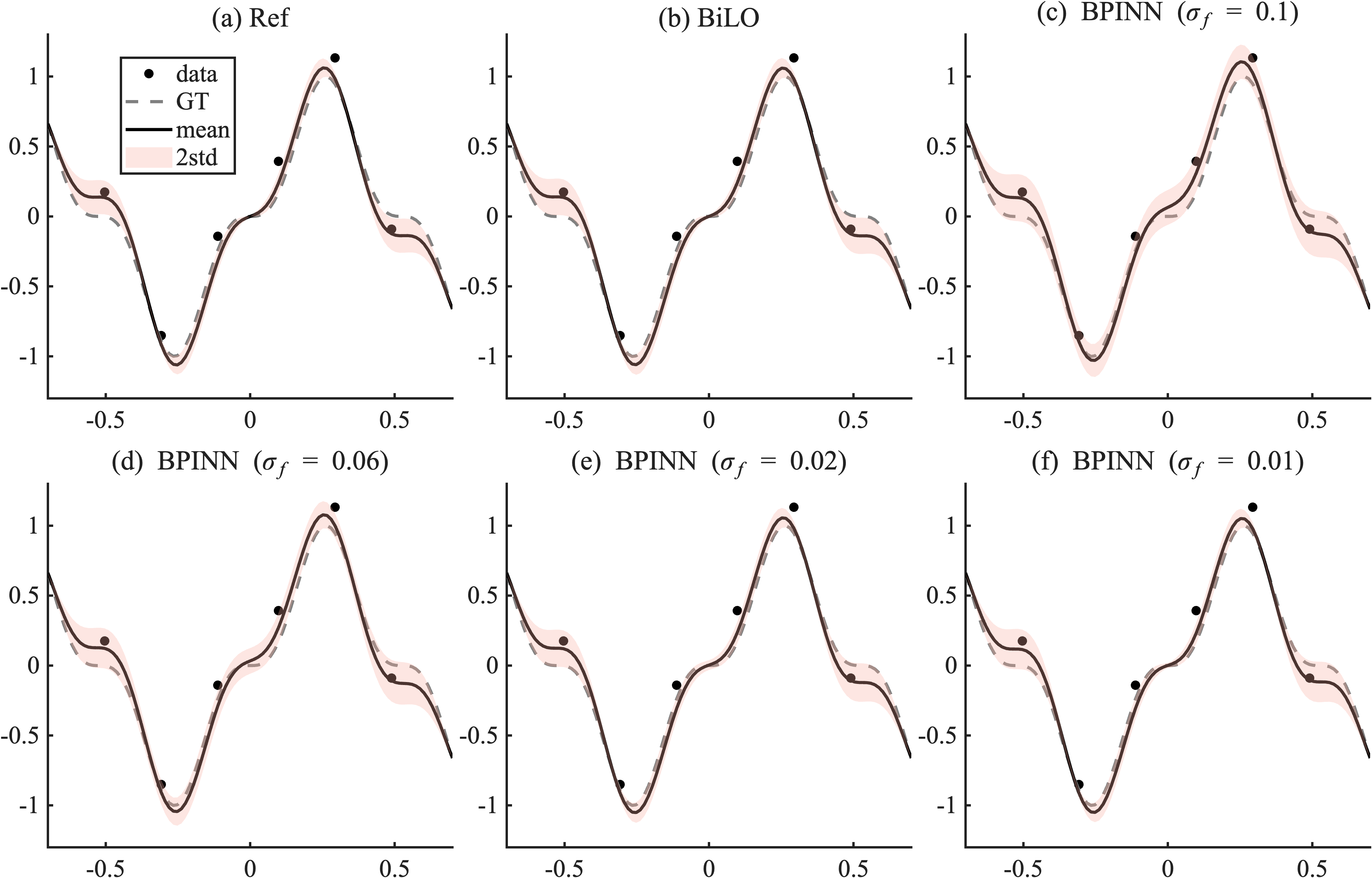}
  \caption{
    Inference results for the nonlinear Poisson problem using B-BiLO and BPINN with different noise levels $\sigma_f$.
Each panel shows the noisy data, ground-truth (GT) solution, and the mean and standard deviation of the inferred solution $u$.
(a) Reference solution obtained by Metropolis–Hastings (MH) sampling with a numerical PDE solver.
(b) B-BiLO.
(c–f) BPINN results for various values of $\sigma_f$.
    }
    \label{f:nonlin_examples}
  \end{figure}
Table \ref{tab:comparison} supplements Figure 2 in the main text, and show the the mean and standard deviation of the inferred parameter $k$ and the standard deviation of $u(0)$ using different methods.
\begin{table}[h!]
  \centering
  \begin{tabular}{lccc}
    \toprule
              & mean($k$)     & std($k$)      &  std($u(0)$) \\ \midrule
    BPINN ($\sigma_f = 0.1$)   & 0.746 & 0.043 & $62.4\times10^{-3}$ \\
    BPINN ($\sigma_f = 0.6$)   & 0.747 & 0.033 & $46.0\times10^{-3}$ \\
    BPINN ($\sigma_f = 0.01$)   & 0.741 & 0.034 & $9.5\times10^{-3}$ \\
    BiLO        & 0.753 & 0.026 & $3.0\times10^{-3}$ \\
    Reference          & 0.754 & 0.025 & $5.0\times10^{-15}$ \\
    \bottomrule
  \end{tabular}
  \caption{Comparison of the mean and standard deviation of $k$ and the standard deviation of $u(0)$ using different methods, the latter of which should be 0.
  }
  \label{tab:comparison}
\end{table}

\subsection{Inferring Patient-Specific Tumor Growth Parameters from MRI Data}

The setup for the glioblastoma (GBM) tumor growth model is detailed in \cite{zhangPersonalizedPredictionsGlioblastoma2025}.
The patient brain geometry is obtained by diffeomorphic registration of an anatomical atlas to the patient’s MRI, yielding tissue-dependent distributions of white and gray matter, with the diffusion coefficient in white matter assumed to be ten times that in gray matter. 
We assume the initial tumor cell density to be $u_0(\bx) = 0.5\exp(-0.1\lVert \bx - \bx_0 \rVert^2)$, where $\bx_0$ is chosen to be the centroid of the TC segmentation.
We assume the predicted segmentations $\by^s(\bx)$, $s\in\{TC,WT\}$, are related to the tumor cell density $u(\bx,1)$ the tumor cell density $u$ at the nondimensional $t=1$ via a (soft) thresholding operation:
$\by^s(\bx) = s(50(u(\bx,1) - u_c^s))$,
where $s(x) = 1/(1+e^{-x})$ is the sigmoid function. 
To handle the complex brain geometry without meshing, we employ the diffuse-domain method, which embeds the brain into a regular computational domain with a smooth boundary representation. 
To address the ill-posedness of parameter estimation from single-time MRI data, the PDE is non-dimensionalized following the approach we developed in \cite{zhangPersonalizedPredictionsGlioblastoma2025}.
The Patient-specific characteristic parameters are obtained by ignoring the complex brain geometry, resulting in a radially symmetric PDE in 1D, and selecting the parameters whose predicted tumor radii best match the segmented ones.
These characteristic parameters are then used to scale and infer personalized diffusion and proliferation parameters from the MRI segmentations. 
The inferred parameters enable personalized prediction of tumor cell density and infiltration patterns, which we validate using recurrence data. 
Compared with the standard uniform-margin clinical target volume, the model-based personalized prediction achieves comparable or improved coverage of the recurrence with reduced irradiation volume, demonstrating its potential for personalized radiotherapy planning \cite{zhangPersonalizedPredictionsGlioblastoma2025,lipkovaPersonalizedRadiotherapyDesign2019,balcerakIndividualizingGliomaRadiotherapy2025}.

Accurately resolving the complex brain geometry and tumor growth dynamics requires a large number of collocation points, which in turn demands substantial GPU memory.
Owing to the limited GPU capacity, we adopt a finite difference discretization of the residual and the residual gradient to reduce memory usage \cite{limPhysicsInformedNeural2022,chiuCANPINNFastPhysicsinformed2022}.
The model is trained using the Adam optimizer with AMSGrad \cite{kingmaAdamMethodStochastic2017,reddiConvergenceAdam2019} with a learning rate of $10^{-3}$ for the network weights.
The lower-level optimization tolerance is set to $10^{-5}$. 
For the upper-level HMC sampler, we use $L = 10$ leapfrog steps with step size $\delta t = 2\times10^{-3}$ and collect $N=2000$ samples.
The solution network is a six-layer modified MLP \cite{wangUnderstandingMitigatingGradient2021} with 256 neurons per layer, augmented with random Fourier features \cite{wangEigenvectorBiasFourier2021}.
This architecture is chosen for improved optimization convergence; however, BiLO is an algorithmic framework and is not restricted to this particular network design.

We first validate our method on synthetic data generated from a known set of parameters $D=0.894$, $\rho=0.860$, $\ucwt=0.3$, $\uctc=0.5$, $\bar{D}=0.17$, $\bar{\rho}=0.31$.
The prior $\uctc\sim$ Uniform(0.1,0.4), $\ucwt\sim$ Uniform(0.4,0.8), $D\sim$logNormal(0.1,0.316), $\rho\sim$logNormal(0.1,0.316).
The geometry is a 2D slice of the brain atlas. 
The results are shown in Fig.~\ref{f:gbm_synthetic}.

\begin{figure}[htbp]
  \centering
  \includegraphics[keepaspectratio=true, width=\textwidth]{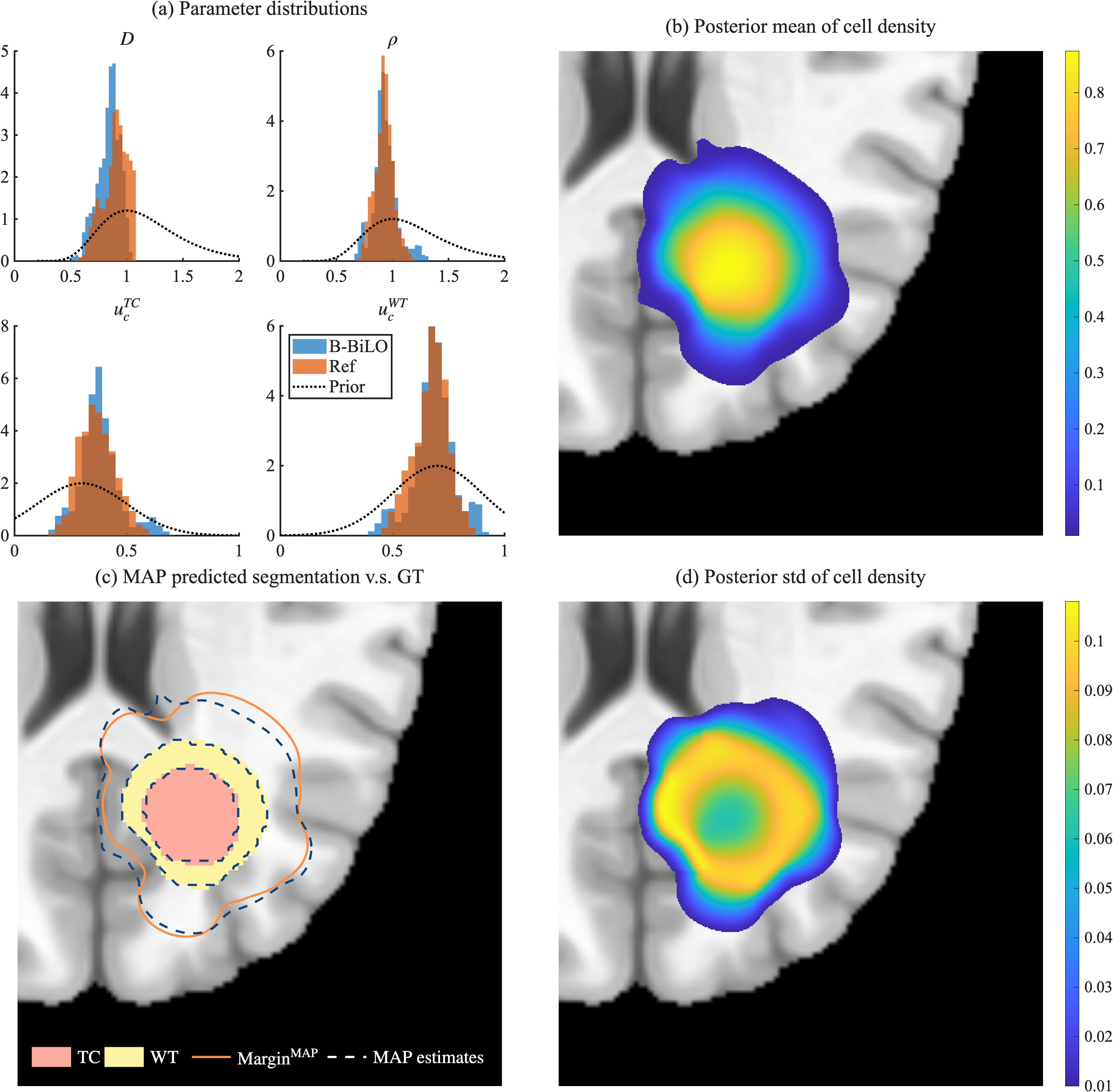}
  \caption{BiLO inference on synthetic tumor.
  (a) The inferred posterior distribution of $D$, $\rho$, $\ucwt$, $\uctc$ and their prior distributions, using the reference method (MH and numerical solver) and B-BILO rank 8. The legends also indicated the posterior mean and standard deviation.
  (b) The posterior mean of the tumor cell density.
  (c) Ground truth WT and TC segmentations (filled) with the infiltration margin (solid orange line), and MAP-predicted segmentations and infiltration margin (dashed lines).
  (d) The posterior standard deviation of the tumor cell density.
  }
  \label{f:gbm_synthetic}
\end{figure}

For the patient case shown in the main text, 
the prior distributions are $\uctc\sim$ Normal$(0.3,0.1^2)$, $\ucwt\sim$ Normal$(0.7,0.1^2)$, $D\sim$logNormal$(0.01,0.1)$, $\rho\sim$logNormal$(0.01,0.1)$.

\subsection{Example 3: Inferring Stochastic Rates from Particle Data}
\label{ss:pointproc}

The PDE in $\Omega = [0,1]$ can be written as a elliptic interface problem:
\begin{equation*} \label{eq:interface}
  \begin{cases}
    \Delta u - \mu u = 0\\
    u(0) = u(1) = 0 \\
    u^+(z) = u^-(z)\\
    u_x^+(z) - u_x^-(z) = -\lambda
  \end{cases}
\end{equation*}
Here, the superscripts $+$ and $-$ indicate the limits from the right and left of the interface point $z$. While the solution $u$ is continuous across this point, its derivative exhibits a jump discontinuity.

To address the singularity in the forcing term, we employ the cusp-capturing PINN method from \cite{tsengCuspcapturingPINNElliptic2023}. 
This technique involves learning a function $\tilde{u}(x,\phi)$ such that $u(x) = \tilde{u}(x,|x-z|)$.
This formulation inherently satisfies the continuity condition. The jump condition is incorporated as an additional constraint within $\cF$: 
$$\partial_\phi \tilde{u}(z,0) = -\lambda.$$ 
The cusp-capturing PINN is parameterized by $\tilde{u}(x, \phi; \wnn)$, and ``jump loss'' is needed to enforce the jump condition:
\begin{equation*}
  \mathcal{L}_{\rm jump}^{PINN}(W) = (\partial_\phi \tilde{u}(z, 0; \wnn) + \lambda)^2
\end{equation*}

In our BiLO framework, we parameterize the local operator as $\tilde{u}(x, \phi, \theta; \wnn)$, with $\theta = (\lambda, \mu)$. 
The corresponding "jump loss" is:
\begin{equation*}
  \mathcal{L}_{\rm jump}(\theta,W) = (\partial_\phi \tilde{u}(z, 0, \theta; \wnn) + \lambda)^2.
\end{equation*}
Additionally, we define a ``jump gradient loss" to satisfy the local operator conditions:
\begin{equation*}
  \mathcal{L}_{\rm jgrad}(\theta, W) = \left\|\nabla_\theta \partial_\phi \tilde{u}(z,0,\theta; \wnn)\right\|^2
\end{equation*}
The lower-level local operator loss is given by ${\mathcal L}_{LO} = \mathcal{L}_{\rm res} + \mathcal{L}_{\rm jump} + \wdr (\mathcal{L}_{\rm rgrad} + \mathcal{L}_{\rm jgrad})$.

For the experiments in this section, we use a 4-layer residual fully connected neural network \cite{heDeepResidualLearning2015} with width 512 and $\tanh$ activations, augmented with random Fourier features \cite{wangEigenvectorBiasFourier2021}. The model is trained using the Adam optimizer \cite{kingmaAdamMethodStochastic2017} with AMSGrad \cite{reddiConvergenceAdam2019}, employing a learning rate of $10^{-4}$ for the network. 
We collect 5,000 HMC samples with $\delta t = 10^{-2}$ and $L$ = 50, yielding ESS of approximately 150.
The exact solution to the PDE is given by \cite{milesInferringStochasticRates2024}:
\begin{equation*}
  u(x) = \frac{\lambda}{\sqrt{\mu}} \csch(\sqrt{\mu})\sinh(\sqrt{\mu} \min\{x,z\})\sinh(\sqrt{\mu} (1 - \max\{x,z\}))
\end{equation*}
Due to the relatively large magnitude of the solution, We set $\epsilon = 10$, which results in a relative error of approximately 1\% in the infinity norm compared to the exact solution.
The local operator loss is evaluated on 101 evenly spaced points in the domain.
In this problem, the decay rate $\mu$ is typically on the order of $10$, while the birth rate $\lambda$ is on the order of several hundreds, as determined by the biological dynamics of the system. This scale ensures that a sufficient number of particles are present for inference.
To improve numerical conditioning during training, we reparameterize $\lambda = 100 \bar{\lambda}$ and learn the rescaled parameter $\bar{\lambda}$. We also represent BiLO as $u(x, \theta; W) = m(x, \theta; W)^2 x(1 - x)$, where $m$ is the raw output of the MLP. This transformation enforces the boundary conditions $u(0) = u(1) = 0$ and ensures $u \ge 0$, which is necessary for evaluating $\log u$ in the likelihood.

In Table \ref{t:comparison}, we show the relative error of the MAP estimate of $\lambda$ and $\mu$, and the area of 90\% HPDR with respect to the reference method.
\begin{table}[!htbp]
  \centering
  \begin{tabular}{cccc}
  \toprule
    & Area & $\lambda_{\rm MAP}$ & $\mu_{\rm MAP}$ \\ 
    \midrule
    Full FT & 3.5\% & 1.8\% & 3.4\% \\ 
    LoRA rank 4 & 2.1\% & 2.0\% & 4.0\% \\ 
  \bottomrule
  \end{tabular}
  \caption{relative error of the MAP estimates of $\lambda$ and $\mu$ and the area of 90\% HPDR with respect to the reference method}
  \label{t:comparison}
\end{table}

\subsection{Darcy Flow Problem}
\label{ap:ss:darcy}
We set the lower-level optimization tolerance to $\epsilon = 0.5$. 
The neural network is a six-layer residual MLP with 512 neurons per layer \cite{heDeepResidualLearning2015}, augmented with random Fourier features \cite{wangEigenvectorBiasFourier2021}.
For the upper-level HMC sampler, we use $L = 100$ leapfrog steps with step size $\delta t = 10^{-3}$ and collect $N=5000$ samples. The BiLO model is represented as
$u(\bx, z; W) = m(\bx, z; W) \bx_1 (1 - \bx_1) \bx_2 (1 - \bx_2)$, 
where $m$ is the raw output of the MLP and $z$ is the auxiliary variable. This form ensures that the Dirichlet boundary conditions are satisfied by construction. 

We use a KL expansion to represent the unknown function $D(\bx)$.
Let $g(\bx, \theta)$ be the 64 term KL expansion of a mean zero Gaussian process with covariance kernel $C = (-\Delta + \tau^2)^{-d}$, 
where $\tau=3$ and $d=2$, and $-\Delta$ is the Laplacian on $\Omega$ subject to homogeneous Neumann boundary conditions \cite{huangEfficientDerivativefreeBayesian2022,liFourierNeuralOperator2024}:
\begin{equation}
  g(\bx,\theta) = \sum_{i=1}^{8} \sum_{j=1}^{8} \theta_{ij} \sqrt{\lambda_{ij}} \cos(i\pi x_1) \cos(j\pi x_2), \quad \lambda_{ij} = (\pi^2 (i^2 + j^2)+\tau^2) ^{-2}
\end{equation}
The prior of the KL coefficients is the standard normal distribution.
In the example, we assume $D(\bx,\theta) = 9s(20g(\bx, \theta)) + 3$, where $s(x) = 1/(1+e^{-x})$ is the sigmoid function. This models a spatially varying diffusion coefficient with low and high diffusivity (3 and 12) in the domain. 
The initial guess for all simulations is $\mathbf{\theta} = \mathbf{0}$.

In Fig.~\ref{f:darcy2d_bayes_vsref}, we compare the results of B-BiLO with LoRA rank 4 and the reference method, which use 100,000 steps of MH sampling with an numerical solver on fine mesh. Key spatial patterns are well captured, though fine-scale details differ due to the inherent randomness of sampling.
\begin{figure}[!htb]
  \centering
  \includegraphics[keepaspectratio=true, width=\textwidth]{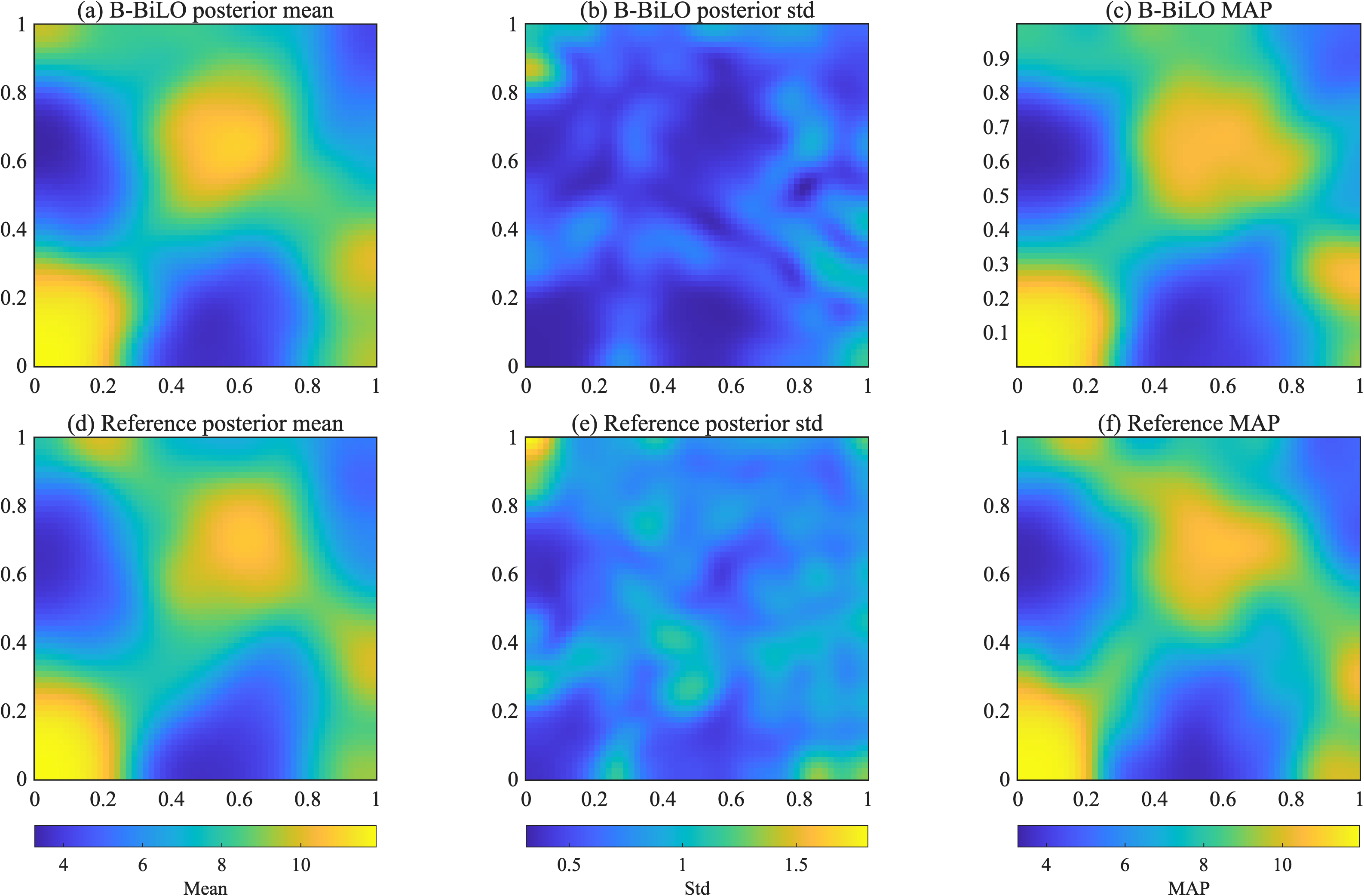}
  \caption{Comparison between B-BiLO with LoRA rank 4 (row 1) and the reference method (MH with numerical solver, row 2) for the 2D Darcy flow problem
  The first columns show the inferred posterior mean of the diffusion coefficient $D(\bx)$.
  The second columns show the posterior standard deviation of $D(\bx)$.
  The third columns show the MAP predicted $D(\bx)$.
  }
  \label{f:darcy2d_bayes_vsref}
\end{figure}

\end{appendices}

\backmatter

\bibliography{bilolora}